%% file: paper_AISTATS2024_CameraReady_476.tex
\documentclass[twoside]{article}

\usepackage[accepted]{aistats2024}
\usepackage[round]{natbib}

\input{Definitions}

\input{myshorts}

\input{psfig.sty}

\usepackage{enumitem}
\begin{document}

\runningtitle{Efficient Graph Laplacian Estimation by Proximal Newton}
\runningauthor{Y. Medvedovsky, E. Treister, T. Routtenberg}

%

%


\twocolumn[
\aistatstitle{Efficient Graph Laplacian Estimation by Proximal Newton}
\aistatsauthor{ Yakov Medvedovsky \And Eran Treister \And Tirza Routtenberg }
\aistatsaddress{Electrical and Computer Engineering,\\ Ben-Gurion University. \And  Computer Science Dept, \\ Ben-Gurion University. \And Electrical and Computer Engineering,\\ Ben-Gurion University. } 
]

\begin{abstract}
  The Laplacian-constrained Gaussian Markov Random Field (LGMRF) is a common multivariate statistical model for learning a weighted sparse dependency graph from given data. This graph learning problem can be formulated as a maximum likelihood estimation (MLE) of the precision matrix, subject to Laplacian structural constraints, with a sparsity-inducing penalty term. This paper aims to solve this learning problem accurately and efficiently.
First, since the commonly used $\ell_1$-norm penalty is inappropriate in this setting and may lead to a complete graph, we employ the nonconvex minimax concave penalty (MCP), which promotes sparse solutions with lower estimation bias. Second, as opposed to existing first-order methods for this problem, we develop a second-order proximal Newton approach to obtain an efficient solver, utilizing several algorithmic features, such as using conjugate gradients, preconditioning, and splitting to active/free sets. 
Numerical experiments demonstrate the advantages of the proposed method in terms of both computational complexity and graph learning accuracy compared to existing methods.
\end{abstract}

\section{{INTRODUCTION}}
Graphs are fundamental mathematical structures used in various fields to represent data, signals, and processes. Weighted graphs naturally arise when processing networked data applications, such as computer, social, sensor, energy, transportation, and biological networks \citep{Newman_2010}, where the data is inherently related to a graph associated with the underlying network. 
As a result, graph learning is used in a wide range of applications, both for virtual, data-based networks 
and for physical networks of infrastructures;
 the latter include the brain \citep{vecchio2017connectome}, financial analysis \citep{giudici2016graphical}, and electrical networks \citep{Deka_Chertkov_Backhaus_2017}.
 In this context, a fundamental graph learning problem is the estimation of the Laplacian matrix, which plays a central role in spectral graph theory \citep{chung1997spectral} and machine learning \citep{von2007tutorial}.
Moreover, in the field of graph signal processing (GSP) \citep{shuman2013emerging},  the Laplacian matrix is used to extend basic signal processing operations, such as filtering \citep{milanfar2012tour}, sampling \citep{anis2016efficient},  and signal recovery \citep{narang2012perfect,kroizer2022routtenberg}. Therefore, accurate and efficient Laplacian learning is extremely valuable for data processing in networks.

Gaussian Markov Random Fields (GMRFs) \citep{rue2005gaussian} are probabilistic graphical models that have been studied extensively over the last two decades \citep{banerjee2006convex,hsieh2013big,hsieh2011sparse}. In particular, there has been a significant effort to develop approaches for learning sparse precision matrices under GMRF models, where this precision represents a graph topology. The graphical LASSO (GLASSO) \citep{friedman2008sparse} is a widely-used approach to estimate the sparse precision matrix using $\ell_1$-norm regularization.
Laplacian-constrained GMRF (LGMRF) models are a specific type of GMRF, in which the precision is a Laplacian matrix.
Laplacian constraints imply that smooth signals,  which have similar values for nodes connected by large weights,  have a higher probability \citep{ying2020nonconvex}. 
Smooth graph signals are common in various network applications, such as power systems \citep{dabush2021state,drayer2018detection}.
Thus, LGMRF models
have been widely studied in semi-supervised learning frameworks \citep{zhu2003combining},  to analyze real-world datasets 
 \citep{de2020learning}, and in GSP \citep{dong2016learning,egilmez2017graph}.  

The estimation of standard and Laplacian-constrained GMRF models are closely related. However, the extension of GLASSO and similar GMRF estimation algorithms to the case of LGMRF, which also considers the Laplacian constraints, is not straightforward. For example, since the Laplacian matrix is sparse in many applications, then, similar to the GLASSO approach, $\ell_1$-norm regularization methods were used to promote the Laplacian sparsity \citep{egilmez2017graph,kumar2019structured,liu2019block,zhao2019optimization}. However, it was recently shown that the $\ell_1$-norm is an inappropriate penalty for promoting the sparsity of the precision matrix under the LGMRF model, as it leads to an inaccurate recovery of the connectivity pattern of the graph
\citep{ying2020does}. 
Thus, to properly address the Laplacian estimation problem, other sparsity-promoting penalty terms should be used \citep{ying2020does,ying2020nonconvex}. 
However, these existing methods are of first-order (i.e., gradient-based), while in general GMRF estimation literature \citep{oztoprak2012newton,hsieh2013big,hsieh2011sparse,treister2014block} it is known that second-order methods, i.e., proximal Newton \citep{lee2014proximal}, yield better performance in terms of run time. In such approaches, a quadratic approximation is applied on the smooth part, while the non-smooth regularization term remains intact. This exploits a connection between the gradient and Hessian of the MLE objective and hence is efficient for GMRF estimation. 
The work of \citet{Moghaddam_2016} employs a second-order approach for Laplacian-constrained topology identification. However, it is based on the convex $\ell_1$-norm penalty, which was found inappropriate for LGMRF \citep{ying2020does}, and the edge weights are not constrained to be nonnegative, which may result in non-Laplacian solutions. 
Therefore, there is a need for second-order methods for Laplacian learning that utilize proper regularization (to yield accurate sparsity patterns) and are efficient in run time. 

\textbf{Contribution.}
In this paper, we present an efficient graph estimation method for the  LGMRF model based on the proximal Newton approach. Our method involves solving a nonconvex penalized MLE problem with MCP penalty, instead of the commonly-used $\ell_1$-norm penalty. At each iteration, the smooth part of the objective is approximated by a second-order Taylor expansion around the current precision matrix, while the non-smooth penalty \textit{and Laplacian constraints} remain intact. The resulting Newton problem is a penalized quadratic minimization under linear constraints to guarantee that the next iterate is a Laplacian matrix. 
To the best of our knowledge,  the proposed method is the first proximal Newton method for the LGMRF model estimation. Furthermore, our framework includes two nontrivial algorithmic novelties in this context. 
First, the inner, constrained, Newton problem is solved by a projected nonlinear conjugate gradient method. This yields a significant computational speedup over gradient-based first-order methods. Second, we introduce a diagonal preconditioner to improve the performance further. We show some theoretical results regarding the problem and the algorithms we propose.  Furthermore, we demonstrate the effectiveness of the proposed method in learning sparse graphs through numerical experiments.

\textbf{Organization and notation.}
The rest of this paper is organized as follows. In Section \ref{Prob_for_sec}, we describe the considered  Laplacian learning problem, simplify its formulation, and discuss the sparsity-promoting penalty function.
In Section \ref{MCP_GLN_sec}, we develop the proposed NewGLE method. 
Simulations are shown in Section \ref{Simulations_sec} and
 the paper is concluded in Section \ref{conc}.

In the rest of this paper we denote vectors and matrices in boldface lowercase and uppercase letters, respectively.
 The elements of the vectors $\onevec$ and $\zerovec$ are ones and zeros, respectively. 
The notations $\|\cdot\|_{\textbf{F}}$, $|\cdot|$, $\otimes$, $(\cdot)^{T}$, $(\cdot)^{-1}$, $(\cdot)^{\dagger}$, $|\cdot|_+$, and  $\text{Tr}(\cdot)$ 
denote the Frobenius norm, determinant operator, Kronecker product, transpose, inverse, Moore-Penrose pseudo-inverse, pseudo-determinant, 
and trace, respectively. The sets
$\mathcal{S}^p$ and $\mathcal{S}^p_+$ are of real symmetric and real positive semi-definite  $p \times p$ matrices, respectively.
For matrix $\Amat\in\mathbb{R}^{p \times p}$, $\nabla_{\Amat}f\in \mathbb{R}^{p \times p}$ and $\nabla_{\Amat}^2f\in\mathbb{R}^{p^2 \times p^2}$ are the gradient and Hessian of the scalar function $f(\Amat)$. 

\section{PROBLEM FORMULATION}
\label{Prob_for_sec}
In this section, we formulate the considered  Laplacian learning problem.
We start with definitions related to GSP. Then, we present the LGMRF and formulate the graph estimation problem under the LGMRF and MLE in Subsection \ref{GMRF_subsection}. 
The sparsity-promoting penalty function is discussed in Subsection \ref{sparsity_subsection}.

We consider an undirected, connected, and weighted graph $\mathcal{G}\braro{\mathcal{V},\mathcal{E}}$, where 
$\mathcal{V} = \bracu{1,\dots,p}$ and $\mathcal{E} = \bracu{1,\dots,m}$
are the set of vertices and the set of edges, respectively. The nonnegative weighted adjacency matrix of the graph, $\Wmat$, has the $(i,j)$-th element $w_{i,j}> 0$ if vertex $i$ and $j$ are connected, and zero otherwise. 
From spectral graph theory \citep{chung1997spectral}, the rank of the Laplacian matrix for a connected graph with $p$ nodes is $p-1$. Thus, the set of  Laplacian matrices for connected graphs can be defined as \citep{ying2020nonconvex}
\beqna
\label{Lc}
    \mathcal{L}=\left\{
    \Lmat \in \mathcal{S}^p_+ \Big{|} \begin{array}{ll} 
    \Lmat_{i,j}\leq 0, \forall i \neq j,~i,j=1,\ldots,p\\ 
    \Lmat\onevec=\zerovec,~~{\textrm{rank}}\braro{\Lmat} = p-1
    \end{array}\right\}.
\eeqna

\subsection{GMRF and Constrained MLE}
\label{GMRF_subsection}
In this paper, we use the formulation of the graph estimation problem based on the probabilistic
graphical model \citep{koller2009probabilistic,banerjee2008model}.
We assume that the data samples are obtained from a zero-mean Gaussian distribution parametrized by a positive semi-definite precision matrix, $\Lmat$, i.e.,  $\xvec \sim \mathcal{N}\braro{\zerovec,\Lmat^\dagger}$, where $\mathcal{N}\braro{\boldsymbol{\mu},\boldsymbol{\Sigma}}$ denotes the normal distribution with mean $\boldsymbol{\mu}$ and covariance matrix $\boldsymbol{\Sigma}$.
This defines an improper LGMRF model \citep{kumar2019structured} with parameters $\braro{\zerovec,\Lmat}$, where the precision matrix is $\Lmat \in \mathcal{L}$ and $\mathcal{L}$ is defined in \eqref{Lc}. 

Given $n$  independent and identically distributed (i.i.d.) samples $\xvec_1, \ldots, \xvec_n$ drawn from an LGMRF, our goal is to find the constrained MLE of $\Lmat$ based on these samples and
 under the Laplacian constraints. The MLE  can be found by solving
 the following constrained minimization of the negative log-likelihood \citep{ying2021minimax}:
\begin{equation}
\label{ML}
    \hat{\Lmat}_{MLE} =  
    \underset{\Lmat\in\mathcal{L}}{\text{argmin}}\bracu{\TR{\Lmat\Smat} - \log|\Lmat|_+},
\end{equation}
where $\Smat \define \frac{1}{n}\sum_{i=1}^{n}\xvec_i\xvec^{T}_i$ is the empirical covariance matrix.


Sparsity plays an important role in high-dimensional learning, which helps avoid over-fitting and improves the identification of the relationships among data, especially where $\Smat$ is low rank. In particular, the Laplacian is a sparse matrix in various applications, e.g. electrical networks \citep{Halihal_Routtenberg2022,grotas2019power}.
A sparse graph estimation problem under the LGMRF model can be formulated by adding a sparse penalty function to the estimator in \eqref{ML}, which results in 
\beqna 
\label{minimization_problem_3}
    \begin{aligned}
        & \underset{\Lmat\in\mathcal{L}}{\text{minimize}}
        &\TR{\Lmat\Smat} - \log|\Lmat|_+ +\rho\braro{\Lmat;\lambda},
    \end{aligned}
\eeqna
where $\rho\braro{\Lmat;\lambda}$ is a general sparsity-promoting penalty function and $\lambda$ is a tuning parameter.

The objective function in \eqref{minimization_problem_3} (or in \eqref{ML}) involves the pseudo-determinant term, since the target matrix $\Lmat$ is a singular matrix. 
As a result, the optimization problem becomes hard to solve  
\citep{holbrook2018differentiating}. 
As a remedy to this problem,
it is demonstrated in \citep{egilmez2017graph} that if $\Lmat \in \mathcal{L}$, then we can use the relation
\be
\label{pseudoL}
|\Lmat|_+ = |\Lmat+\Jmat|,
\ee
where $\Jmat = \frac{1}{p}\onevec\onevec^T$.
By substituting \eqref{pseudoL} in \eqref{minimization_problem_3}, the problem is simplified.  
In addition, in the following theorem we show that the rank and the positive semi-definiteness constraints in \eqref{Lc} are redundant. The proof appears in Appendix \ref{sec:ProofThm1}.
\begin{theorem}
\label{Th1}
The optimization problem 
\begin{equation}
\begin{aligned}
\label{minimization problem 4}
    \underset{\Lmat\in \mathcal{F}}{\text{minimize}}
     \quad & \TR{\Lmat \Smat} -\log|\Lmat+\Jmat|
        +\rho\braro{\Lmat;\lambda},
\end{aligned}
\end{equation}
where the feasible set is given by  
\beqna \label{Lc - 9}
\mathcal{F}=\left\{
    \Lmat \in \mathcal{S}^p \Big{|} \begin{array}{ll} 
    \Lmat_{i,j}\leq 0, \forall i \neq j,~i,j=1,\ldots,p\\ \Lmat\onevec=\zerovec
    \end{array}\right\}.
\eeqna
is equivalent to the 
optimization problem stated in \eqref{minimization_problem_3}.
\end{theorem}

 \subsection{Sparsity Promoting Penalty Functions}
 \label{sparsity_subsection}
 Recent works like \citep{egilmez2017graph}
 introduced the $\ell_1$-norm penalized MLE under the LGMRF model to estimate a sparse graph. 
 The choice of $\rho$ in \eqref{minimization problem 4} as $\|\Lmat\|_{1,\text{off}}$, the absolute sum of all off-diagonal elements of $\Lmat$, 
 results in the same objective function as in the well-known GLASSO problem 
but where the feasible set is restricted to $\Lmat\in \mathcal{F}$.
 However, unlike in GLASSO, under Laplacian constraints the $\ell_1$-norm is ineffective in promoting a sparse solution. In particular, it is shown by \citet{ying2020does} that 
   substituting $\rho\braro{\Lmat;\lambda}=\lambda\|\Lmat\|_{1,\text{off}}$ in the objective of \eqref{minimization problem 4} and
 choosing the penalty parameter $\lambda$ to be large enough 
 leads to a fully-connected (complete) graph as the solution. 
Moreover, \citet{ying2020does} show in their numerical experiments that the number of the estimated positive edges increases as  $\lambda$ increases. An intuitive explanation for this phenomenon is that the $\ell_1$-norm promotes sparsity uniformly on all off-diagonal entrances of the matrix. 
Hence, for $\lambda>0$, minimizing the $\ell_1$-norm separately with a positive definite constraint, as in GLASSO, leads to a diagonal matrix with a positive diagonal. However, unlike a regular GMRF model, the Laplacian constraints, $\Lmat\onevec = \zerovec$ and $\Lmat_{i,j}\leq 0,$ for $i\neq j$, relate the off-diagonal entries with the diagonal entries of $\Lmat$. Hence, minimizing the $\ell_1$-norm would lead to the zeros matrix.
 Therefore, as $\lambda$ increases, the eigenvalues of $\Lmat$  get closer to zero, and the $-\log|\cdot|$ term in the objective function increases significantly. 
Since $\ell_1$-norm penalizes all off-diagonal elements in a uniform way, to prevent a significant increase of the objective function, the minimization process inserts relatively small (in absolute value) negative numbers, even for zero elements in the ground-truth Laplacian matrix. This way, both the $\ell_1$-norm and the expression $-\log|\cdot|$ in the optimization problem remain relatively small.

 
Alternative approaches using nonconvex penalties, such as clipped absolute deviation (SCAD) \citep{fan2001variable} and the minimax concave penalty (MCP) \citep{zhang2010nearly}, have been proposed to alleviate this issue. 
These penalty functions penalize the elements in the precision matrix more selectively than the $\ell_1$-norm, such that elements that receive significant values from the MLE will receive insignificant penalties, while elements closer to zero will receive a greater penalty.
In this paper, we use the MCP penalty function, which is applied on the off-diagonal entries of $\Lmat$ as follows:
\be
\label{r mcp}
\rho_{MCP}\braro{\Lmat;\lambda} = \sum\nolimits_{i\neq j}MCP\braro{\Lmat_{ij};\gamma,\lambda},
\ee
where the $MCP(\cdot)$ is a scalar function, defined as
\beqna
    \label{MCP}
    MCP(x;\gamma,\lambda)
    =
    \left\{\begin{array}{lr}
    \lambda|x| - \frac{x^2}{2\gamma} \quad \textrm{if}\quad |x| \leq \gamma\lambda\\
    \frac{1}{2}\gamma\lambda^2 \quad \quad \quad \textrm{if}\quad |x| > \gamma\lambda
    \end{array}\right..
\eeqna
The function is constant and has a derivative of 0 for values larger than $\gamma\lambda$ in magnitude, hence only elements smaller than $\gamma\lambda$ are affected by $\rho_{MCP}()$.

To conclude, our goal in this paper is to develop an efficient algorithm for finding the constrained MLE of the Laplacian matrix, which is the precision matrix of the LGMRF model. This is done by solving the optimization problem in \eqref{minimization problem 4} with $\rho\braro{\Lmat;\lambda}=\rho_{MCP}\braro{\Lmat;\lambda}$, defined in \eqref{r mcp}. I.e., the problem we solve is given by:
 \beqna 
\label{minimization problem final}
    \begin{aligned}
        & \underset{\Lmat\in \mathcal{F}}{\text{minimize}}
        &   \TR{\Lmat\Smat} - \log|\Lmat+\Jmat|+\rho_{MCP}\braro{\Lmat;\lambda}\\
    \end{aligned},
\eeqna
where $\mathcal{F}$ is defined in \eqref{Lc - 9}.

\section{METHOD}
\label{MCP_GLN_sec}
In this section, we develop the proposed proximal Newton method for the graph Laplacian estimation, named NewGLE, that aims to solve \eqref{minimization problem final} efficiently using a proximal Newton approach adopted to the problem. First, we define the constrained Newton optimization problem for the graph learning in Subsection \ref{prox newton_subsection}, keeping the Laplacian constraints when finding the Newton direction. Then, in Subsection \ref{parameterization_subsection}, we present the approach that maps a Laplacian matrix to a vector and simplifies the set of constraints. Next, we describe the idea of restricting the Newton direction into a ``free set" in Subsection \ref{free set_subsection}. 
This constitutes the NewGLE method, which is summarized in Algorithm 1.
Following that, we present our inner solver for finding the Newton direction, which is composed of the nonlinear projected conjugate gradient that is used together with a diagonal preconditioner, in Subsection \ref{CG_subsection} and Appendix \ref{sec:NLCG}. Finally, we provide an algorithm (Algorithm \ref{algorithm for Newton problem} in the Appendix) for finding the Newton direction.

\subsection{Proximal Newton for Graph Laplacian Estimation}
\label{prox newton_subsection}
In the ``proximal Newton” approach, a quadratic approximation is applied on the smooth part of the objective function, while leaving the non-smooth term intact, in order to obtain the Newton descent direction. This approach is considered to be among the state-of-the-art methods for solving the GLASSO problem \citep{rolfs2012iterative,hsieh2013big,hsieh2011sparse,mazumder2012exact,treister2016multilevel}, which is highly related to our problem in \eqref{minimization problem final}. The advantage of this method lies in the treatment of the $\log\det$ term (as we show next), which appears in both problems. However, the GLASSO methods cannot be applied to our graph learning problems because we aim to estimate a precision matrix that satisfies the Laplacian constraints (and is singular). In contrast, the learned precision matrix under the GMRF model in GLASSO is a general positive definite matrix. In addition, we consider the MCP penalty, while the above mentioned methods consider the $\ell_1$-norm penalty, which is not suitable for our case, as discussed in Subsection \ref{sparsity_subsection}.

In our proximal Newton approach we design a constrained Newton problem, where the smooth part of the objective function in \eqref{minimization problem final} is 
\be
\label{eq:smooth_f}
f\braro{\Lmat} = \TR{\Lmat \Smat} -\log{|\Lmat+\Jmat|},
\ee
and the penalty function, $\rho_{MCP}\braro{\Lmat;\lambda}$, is the non-smooth term. At the $t$-th iteration of the proximal Newton approach, the smooth part of the objective is approximated by a second-order Taylor expansion around the current estimation $\Lmat^{(t)}$.
To this end, we use the gradient and Hessian of $f(\Lmat)$ that are given by \citep{hsieh2014quic,hessian_gradient}
\be
\label{grad and h}
    \nabla_{\Lmat} f = \Smat - \Qmat, \quad \nabla_{\Lmat}^2f = \Qmat\otimes\Qmat,
\ee
where $\otimes$ is the Kronecker product and $\Qmat=\braro{\Lmat+\Jmat}^{-1}$.

The gradient in \eqref{grad and h} already shows the main difficulty in solving the optimization problem in \eqref{minimization problem final}: it contains $\Qmat$, the inverse of the matrix $\Lmat +\Jmat$, which 
is expensive to compute. The advantage of the proximal Newton approach for this problem is the low overhead: by calculating $\Qmat$ in $\nabla_{\Lmat} f(\Lmat)$, we also get part of the Hessian computation at the same cost.

In the same spirit of the proximal Newton approach, in addition to the penalty term, we also keep the Laplacian constraint intact in the Newton problem. That is, at iteration $t$, the Newton direction $\Delta^{(t)}$ 
solves the constrained penalized quadratic minimization problem:
\begin{equation}
\begin{aligned}
\label{Newton Problem}
    \underset{\Delta \in \mathcal{S}^{p}}{\text{minimize}}
     \quad & f(\Lmat^{(t)})+ \TR{\Delta\braro{\Smat - \Qmat}} + \frac{1}{2}\TR{\Delta\Qmat\Delta\Qmat}\\
     \quad & + \frac{1}{2}\|\varepsilon\cdot\Delta\|_{F}^2 +\rho_{MCP}(\Lmat^{(t)}+\Delta;\lambda)\\
    \textrm{s.t.}
    \quad & \Lmat^{(t)}+\Delta \in \mathcal{F},
\end{aligned}
\end{equation}
where $\Qmat=\braro{\Lmat^{(t)}+\Jmat}^{-1}$. 
It should be noted that for the sake of simplicity, we use the notation 
$\Qmat$  instead of $\Qmat^{(t)}$, although $\Qmat$ is updated with the iteration $t$.
Note that the gradient and Hessian of $f(\cdot)$
in \eqref{grad and h} at the iterate $\Lmat^{(t)}$ are featured in the second and third terms in \eqref{Newton Problem}, respectively. The first term of \eqref{Newton Problem} is constant and can be ignored. Altogether, the objective function in \eqref{Newton Problem} is quadratic with the MCP sparsity-promoting penalty. In addition, we have a constraint to guarantee that the next update of $\Lmat$ is feasible, i.e., $ \braro{\Lmat^{(t)}+\Delta} \in \mathcal{F}$. 
The fourth term, with a small and symmetric weight matrix $\varepsilon \geq 0$ penalizes the squares of the off-diagonals of $\Delta$. This term is used for the stabilization of the iterative solution of \eqref{Newton Problem}, and to make sure that the Hessian of \eqref{Newton Problem} is positive definite, since $\rho_{MCP}$ is concave---more details are given later on in the analysis in Appendix \ref{sec10:decrease}. The advantage of solving \eqref{Newton Problem} with $\rho_{MCP}$ intact, as opposed to majorizing $\rho_{MCP}$ (e.g., using its gradient instead) is that the underlying gradient-based iterative solver of \eqref{Newton Problem} has access to the curvature of $\rho_{MCP}$ and to its behavior in the sub-sections in its definition \eqref{MCP}.

Once the Newton problem is approximately solved, yielding the direction $\Delta^{(t)}$, it is added to $\Lmat^{(t)}$ employing a linesearch procedure to sufficiently reduce the objective in \eqref{minimization problem final}. To this end, the updated iterate is 
\begin{equation}
\label{eq:linesearch}
\Lmat^{(t+1)} = \Lmat^{(t)} + \alpha^{(t)} \Delta^{(t)},
\end{equation}
where the parameter $\alpha^{(t)}$ is obtained by  a backtracking linesearch using Armijo’s rule \citep{armijo1966minimization}. We state the following theorem on such a proximal quadratic solution process, where we approximate the true Hessians by SPD matrices $\Mmat^{(t)}$ whose eigenvalues are bounded by $\lambda^{\Mmat}_{min}$ and $\lambda^{\Mmat}_{max}$ from below and above, respectively. The proof appears in Appendix \ref{app:convergence}.
\begin{theorem}\label{Thm2}
Let $\{\Lmat^{(t)}\}$ be a series of points produced by a sequence of proximal quadratic minimizations, constrained by $\mathcal{F}$, with some SPD matrices $0\prec\lambda^{\Mmat}_{min}\Imat\preceq \Mmat^{(t)} \preceq \lambda^{\Mmat}_{max}\Imat$ as the approximate Hessian (as in \eqref{eq:minimization problem G}), followed by the linesearch \eqref{eq:linesearch}, starting from $\Lmat^{(0)}\in\mathcal{F}$. Then any limit point $\bar \Lmat$ of $\{\Lmat^{(t)}\}$ is a stationary point of \eqref{minimization problem final}.
\end{theorem}
The theorem shows that if we choose $\Mmat^{(t)}$ such that $\lambda_{min}^{\Mmat} > \gamma^{-1}$, then the method monotonically converges to a stationary point of \eqref{minimization problem final}. Note that $\Mmat^{(t)}$ are matrices we choose for the solution, and are not part of the problem, so this requirement can easily be fulfilled. In the next section, we simplify problem \eqref{Newton Problem} to ease the treatment of the constraints. 

\subsection{Laplacian parameterization}
\label{parameterization_subsection}
To simplify the Laplacian structural constraints in \eqref{Newton Problem}, we use a linear operator that maps a vector $\wvec \in \mathbb{R} ^{p(p-1)/2}$ to a matrix $\Lx{\wvec}\in \mathbb{R}^{p\times p}$. 
\begin{definition}
\label{DEF2}
   The linear operator $\mathcal{P}: \mathbb{R} ^{p(p-1)/2} \mapsto \mathbb{R}^{p\times p}$ is defined as \citep{ying2020does}
   \beqna
    \label{Lx matrix}
    \Lx{\wvec}_{i,j} =
    \left\{\begin{array}{lr}
    -w_k & i>j\\
     \Lx{\wvec}_{j,i} & i<j\\
     -\sum_{j\neq i} \Lx{\wvec}_{i,j} & i=j
    \end{array},\right.
\eeqna
where $k = i-j + \frac{j-1}{2}\braro{2p-j}$ for $i>j$.
\end{definition}
The operator $\mathcal{P}$ defines a linear mapping from a {\em{nonnegative}} weight vector
$\wvec\geq 0$ to a Laplacian matrix $\Lx{\wvec}\in\mathcal{F}$, defined in \eqref{Lc - 9}.
Using this parametrization, we rewrite the optimization problem in \eqref{Newton Problem} such that the decision variable is a vector instead of a matrix. 
In particular, in Appendix \ref{sec:paramNewton}, 
we show that \eqref{Newton Problem} is equivalent to 
the following simpler vector-parameterized Newton problem that we solve 
\beqna
\begin{aligned}
    \underset{\delta\in \mathbb{R} ^{p(p-1)/2}}{\text{minimize}}
     \quad & 
     f_{N}(\delta) \define \TR{\Lx{\delta}\braro{\Smat - \Qmat}}\;  \\ \quad & 
      +\frac{1}{2}\TR{\Lx{\delta}\Qmat\Lx{\delta}\Qmat} +\|\tilde\varepsilon\odot\delta\|_2^2\\\quad & + 2\tilde{\rho}_{MCP}\braro{\wvec+\delta;\lambda}\\
    \textrm{s.t.}
    \quad & \delta \geq -\wvec, \label{Newton Problem 2}
\end{aligned}
\eeqna
where $\tilde\varepsilon\in\mathcal{R}^{p(p-1)/2}$ is the vectorized upper triangle of $\varepsilon$ (without the diagonal) and $\tilde{\rho}_{MCP}$ is given by
\be
\label{rho_def}
\tilde{\rho}_{MCP}\braro{{\wvec};\lambda} \define \sum_{k=1}^{p\braro{p-1}/2}MCP\braro{w_{k};\gamma,\lambda}
\ee
and $ MCP(x;\gamma,\lambda)$ is defined in \eqref{MCP}.
The only remaining constraint in \eqref{Newton Problem 2} is an inequality constraint, which simplifies the projection onto the feasible space.

\subsection{Restricting the updates to free sets}
\label{free set_subsection}
To ease the minimization of \eqref{minimization problem final} we limit the minimization of the Newton problem to a ``free set'' of variables,
while keeping the rest as zeros. This idea was suggested in \citep{hsieh2011sparse} for the GLASSO problem that includes the $\ell_1$-norm penalty. The free set of a matrix $\Lmat$ is defined as
\beqna
\label{free set}
\textbf{Free}\braro{\Lmat} = \bracu{k: \Lmat_{i,j}\neq0 \lor \brasq{\Smat - \Qmat}_{i,j}>\lambda },
\eeqna
where $\lor$ denotes the logistic OR operator and
the relation between $(i,j)$ and $k$ appears in Definition \eqref{DEF2}.
This set comes from 
the subgradient of the $\ell_1$-norm term, in addition to the non-positivity constraint---the zero elements that are not in $\textbf{Free}\braro{\Lmat}$ will remain zero after a (projected) coordinate descent update. Because the subgradient of the MCP penalty is identical to that of $\ell_1$-norm at zero, the free set is defined in the same way as for the $\ell_1$-norm. As $\Lmat^{(t)}$ approaches a solution $\bar\Lmat$, $\textbf{Free}\braro{\Lmat^{(t)}}$ approaches $\bracu{(i,j): \bar\Lmat_{i,j} = 0}$.
Restricting \eqref{Newton Problem 2} to the free set variables improves the Hessian's condition number and decreases the number of iterations needed to solve \eqref{Newton Problem 2}.

So far, we have discussed the outer iteration of the proposed method, which is summarized in Algorithm \ref{alg:OuterAlg}.

\begin{algorithm}[hbt]
    \SetAlgoLined
    \KwInput{Empirical covariance matrix $\Smat$}
            
    \KwResult{Estimated Laplacian matrix $\hat{\Lmat}$}
    
    \KwInit{  Set $\Lmat^{(0)} \in \mathcal{L}$}
    
    \While{Stopping criterion is not achieved}{
         \begin{itemize}[noitemsep,topsep=0.25pt]
            \item Compute $\Qmat = \braro{\Lmat^{(t)} + \Jmat}^{-1}$.\\
            \item Define $\textbf{Free}\braro{\Lmat^{(t)}}$ as in \eqref{free set}.\\
            \item Find the Newton direction $\delta^{(t)}$ by solving \\ \eqref{Newton Problem 2} restricted to $\textbf{Free}\braro{\Lmat^{(t)}}$. \% \textit{by Alg. \ref{algorithm for Newton problem}}\\
            \item  Update: $\Lmat^{(t+1)} = \Lmat^{(t)} + \alpha \Lx{\delta^{(t)}}$, where $\alpha$\\ is achieved by linesearch.
         \end{itemize}
        }
\caption{NewGLE}
\label{alg:OuterAlg}
\end{algorithm}

\subsection{Solution of the Newton problem}
\label{CG_subsection}
To get the Newton direction for the $t$-th iteration, we approximately solve the quadratic, linearly constrained and MCP-regularized problem in \eqref{Newton Problem 2} using an iterative method. We use the nonlinear conjugate gradient (NLCG) method and modify it for our case. Specifically, we use the variant by \citet{dai1999nonlinear}, and since \eqref{Newton Problem 2} is linearly constrained we add projections between the iterations. Moreover, to speed up the convergence, we use NLCG together with a diagonal preconditioner, using the diagonal of the Hessian, as suggested in \citep{zibulevsky2010l1} for the LASSO problem.

The projected and preconditioned NLCG algorithm is the inner and most computationally expensive part of Algorithm \ref{alg:OuterAlg}. The algorithm itself and its details are given in Appendix \ref{sec:NLCG}. 
A key element of the approach is the diagonal preconditioner $\Dmat$, which is applied by replacing the gradient $\braro{\nabla_{\delta} f_N}$ in the search direction by $\Dmat^{-1}\cdot\nabla_{\delta} f_N$. The matrix $\Dmat$
is defined as the diagonal of the Hessian of $f_N(\delta)$, i.e. 
\begin{equation}
\Dmat_{k,k} = \braro{\brasq{\Qmat}_{i,i}+\brasq{\Qmat}_{j,j}-2\brasq{\Qmat}_{i,j}}^2 + \tilde\varepsilon_{k}^2,
\end{equation}
where the relation between $(i,j)$ and $k$ appears in Definition \ref{DEF2}. For the full derivation of the Hessian diagonal, please refer to Appendix \ref{sec:NLCG}. Note that the operator $\Dmat$ is diagonal, so calculating its inverse is trivial.

\section{MORE RELATED WORKS}
\label{related_works_sec}
\textbf{GMRF estimation with $\ell_1$ regularization.}
The $\ell_1$ regularized MLE under the GMRF model has been extensively studied recently (see   \citep{d2008first,friedman2008sparse,rolfs2012iterative,hsieh2013big,hsieh2011sparse,treister2016multilevel,shalom2022pista} and references therein). Some of these methods include a proximal Newton approach. However, these methods cannot be directly applied to the graph learning problem, as the goal here is to learn a Laplacian precision matrix as opposed to a general positive definite matrix. 

\textbf{The $\ell_1$ regularization bias and its alternatives.} One drawback of the mentioned approaches is that the $\ell_1$ penalty causes an estimation bias, which is relevant to several variable selection problems, e.g., least squares minimization and precision estimation under $\ell_1$ priors. To reduce the $\ell_1$ bias, two-stage methods, nonconvex penalties, and hard-thresholding approaches have been introduced for least squares variable selection \citep{zhang2010nearly, breheny2011coordinate, loh2013regularized}  and GMRF estimation \citep{chen2018covariate,lam2009sparsistency,shen2012likelihood,finder2022effective}. 
Several of these methods use nonconvex penalties, such as SCAD, MCP, and capped $\ell_1$-norm  \citep{zhang2010analysis}. Different algorithms have been proposed to solve the resulting nonconvex problems, e.g., minimization-maximization \citep{hunter2005variable}, local quadratic and linear approximations \citep{fan2001variable,zou2008one}.

\textbf{
LGMRF estimation.} The recent works \citep{egilmez2017graph,kumar2019structured,liu2019block,zhao2019optimization} investigate the $\ell_1$-norm penalized MLE under the LGMRF model to estimate a graph Laplacian from data. However, it has been shown recently that imposing an $\ell_1$-norm penalty to the Laplacian-constrained MLE produces an unexpected behavior: the number of nonzero graph weights grows as the regularization parameter increases \citep{ying2020does}. 

To alleviate this issue,
nonconvex penalties, such as
the MCP, have been proposed together with LGMRF estimation.
These methods are first-order (i.e., gradient-based) methods that do not take Hessian information into account, and, thus, are expected to converge slower than second-order methods.
For example, the work of \citet{koyakumaru2022learning} suggests adding Tikhonov regularization to convexify the cost function and use the primal-dual splitting method 
to solve the optimization problem. However, as the Tikhonov regularization parameter increases, the bias of the estimation increases as well. Another approach in \citep{ying2020does,ying2020nonconvex,vieyra2022robust} uses a nonconvex estimation method by solving a sequence of weighted $\ell_1$-norm penalized subproblems, according to a majorization-minimization framework. In each majorization step, the $\ell_1$-norm weights are updated, and then in the minimization step, the $\ell_1$-penalized function is minimized by the projected gradient descent (PGD) method.  
The approach in \citep{ying2021minimax,tugnait2021sparse} requires solving two subproblems:
The first is an initial estimator of the precision matrix, e.g., the MLE. Then, the estimation is used to define the $\ell_1$ weights for the second subproblem. Both stages are solved by PGD. 

All the abovementioned methods are gradient-based methods that do not consider Hessian information.
On the other hand, the experience with the closely related GMRF estimation problem suggests that proximal Newton approaches would result in superior computational efficiency compared to those methods. 

\section{EXPERIMENTAL RESULTS}
\label{Simulations_sec}
In this section, we present experimental results for the evaluation of the proposed method. We define the experimental setup in Subsection \ref{experimental setup_subsection}. Then, we present the results in terms of graph learning performance and computational efficiency in Subsections \ref{ResultsGL_subsection} and \ref{ResultsCC_subsection}, respectively. Our MATLAB code is publicly available on GitHub at 
\url{https://github.com/BGUCompSci/GraphLaplacianEstimationProxNewton}. 

\subsection{Experimental setup}
\label{experimental setup_subsection}
We compare the performance of the following methods:
\begin{itemize}[noitemsep,topsep=0.25pt]
    \item MLE, defined in \eqref{ML}, where the objective function does not include any sparsity-promoting penalty.
    \item Nonconvex Graph Learning (NGL) with MCP \citep{ying2020nonconvex}, which solves a sequence of weighted $\ell_1$-norm penalized subproblems via the majorization-minimization framework. 
    \item Adaptive Laplacian-constrained Precision matrix Estimation (ALPE)  \citep{ying2021minimax}, which is based on a weighted $\ell_1$-norm-regularized MLE.  The weights are precalculated using the MLE. 
    \item PGD for solving \eqref{minimization problem final} using the Projected Gradient Descent method. We developed this method to demonstrate the advantages of using the second-order approximation.  
    \item The proposed method (NewGLE) for the MCP-regularized MLE, which is summarized in Algorithm \ref{alg:OuterAlg} and Algorithm \ref{algorithm for Newton problem} (in Appendix \ref{sec:NLCG}). We used $\gamma=1.01$ and $\varepsilon = 0$, and to start the solution we applied a few PGD iterations, which were taken into account in the results.
\end{itemize}
The regularization parameter, $\lambda$, is fine-tuned for each method for the best performance. 
In addition, all algorithms use the convergence criterion $\frac{\|\Lmat^{(t+1)}-\Lmat^\textrm{(t)}\|_\textrm{F}}{\|\Lmat^{(t+1)}\|_\textrm{F}}\leq \epsilon$ with the tolerance $\epsilon = 10^{-4}$. All the experiments were conducted on a machine with 2 Intel Xeon E5-2660 2.0GHz processors with 28 cores and 512GB RAM.

We use the relative error (RE) and F-score (FS) to evaluate the performance of the algorithms, where
$
\textrm{RE}=\frac{\|\bar{\Lmat} - \Lmat^*\|_{\textrm{F}}}{\|\Lmat^*\|_{\textrm{F}}}$
in which $\bar{\Lmat}$ and $\Lmat^*$ denote the estimated and true precision matrices, respectively,
and $\textrm{FS} = \frac{2\textrm{tp}}{2\textrm{tp}+\textrm{fp} +\textrm{fn}}$,
where the true-positive (tp), false-positive (fp), and false-negative (fn) detection of graph edges are calculated by comparing the supports of $\bar{\Lmat}$ and  $\Lmat^*$. The F-score takes values in $[0, 1]$, where $1$ indicates perfect recovery of the support.

We generated several datasets based on different graph-based models. 
To create the datasets, we first created a graph, and then its associated Laplacian matrix, $\Lmat^*$, is used to generate independent data samples from the Gaussian distribution $\mathcal{N}\braro{\zerovec,(\Lmat^*)^\dagger}$. 
The graph connectivity is based on two options: 1) a random planar graph consisting of $p = 1,000$ nodes \citep{treister2016multilevel}, and 2) a random Barabasi-Albert graph \citep{zadorozhnyi2012structural} of degree 2 with $p = 100$ nodes. The edge weights for both graph models are uniformly sampled from the range $[0.5,2]$. 
The curves in Figs. \ref{fig: RE and vs time p = 1000} and \ref{fig: time vs nDp} are the results of an average of 10 Monte Carlo realizations, and  Fig. \ref{ RE and vs time p = 100} shows the results of an average of 50 
realizations. The standard deviation of all plots in these figures is less than $0.002$, which is also approximately the gap between the two leading methods.

\subsection{Results - graph learning performance}
\label{ResultsGL_subsection}
\begin{figure}[hbt]
    \centering
    \subfloat[\centering]{{\includegraphics[width=0.25\textwidth]{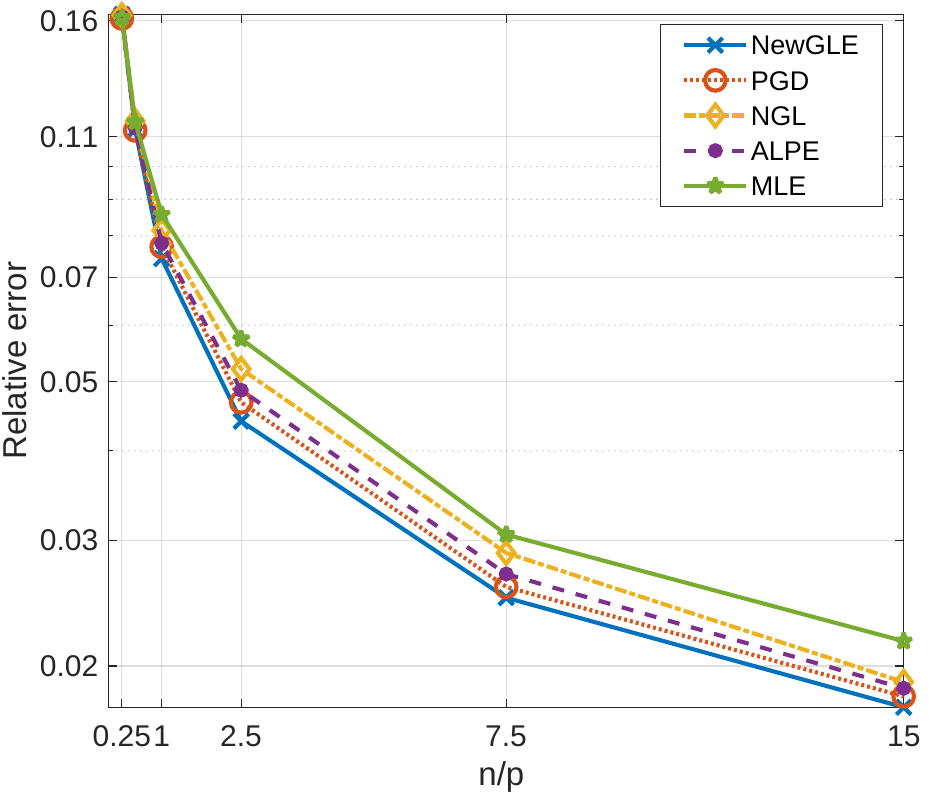} }}%
    \subfloat[\centering]{{\includegraphics[width=0.25\textwidth]{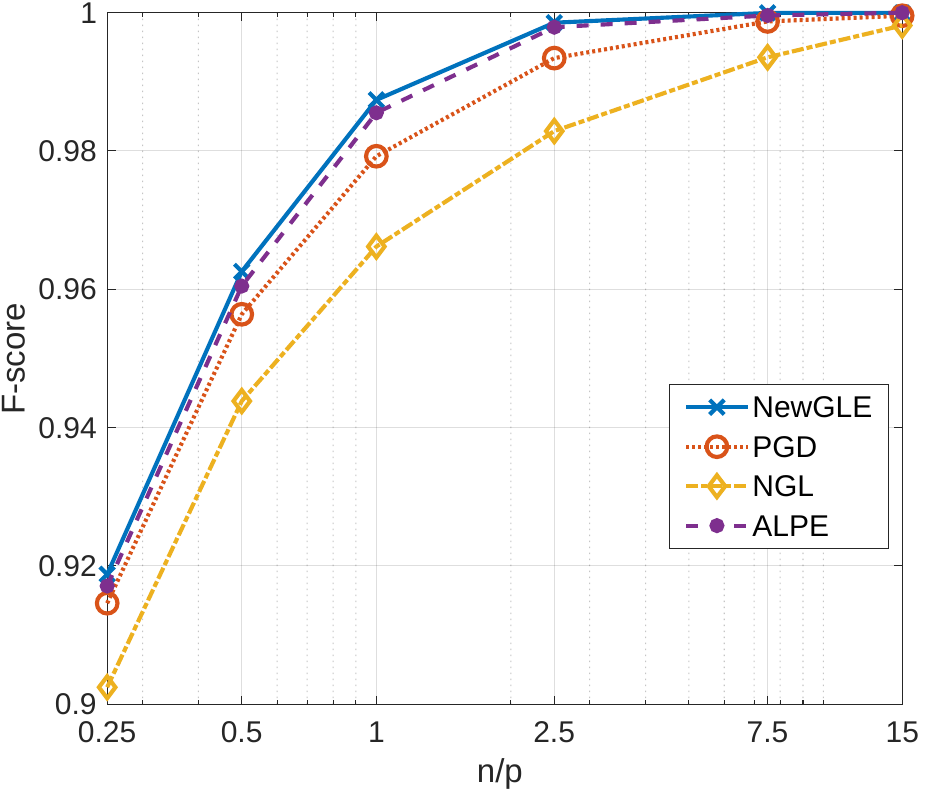} }}
    \caption{Performance comparisons under (a) RE, (b) F-score with different sample size ratios $n/p$  for
    planar graphs with 1,000 nodes.}
    \label{fig: RE and vs time p = 1000}
\end{figure}

\begin{figure}[hbt]
    \centering
    \subfloat[\centering]{{\includegraphics[width=0.245\textwidth]{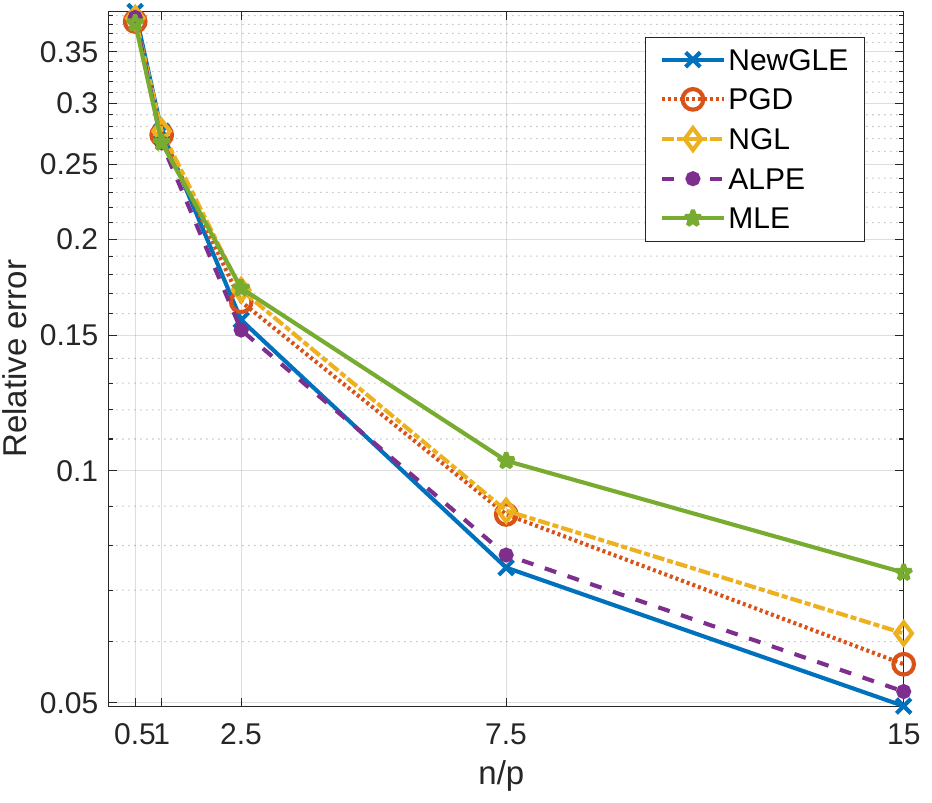} }}
    \subfloat[\centering]{{\includegraphics[width=0.245\textwidth]{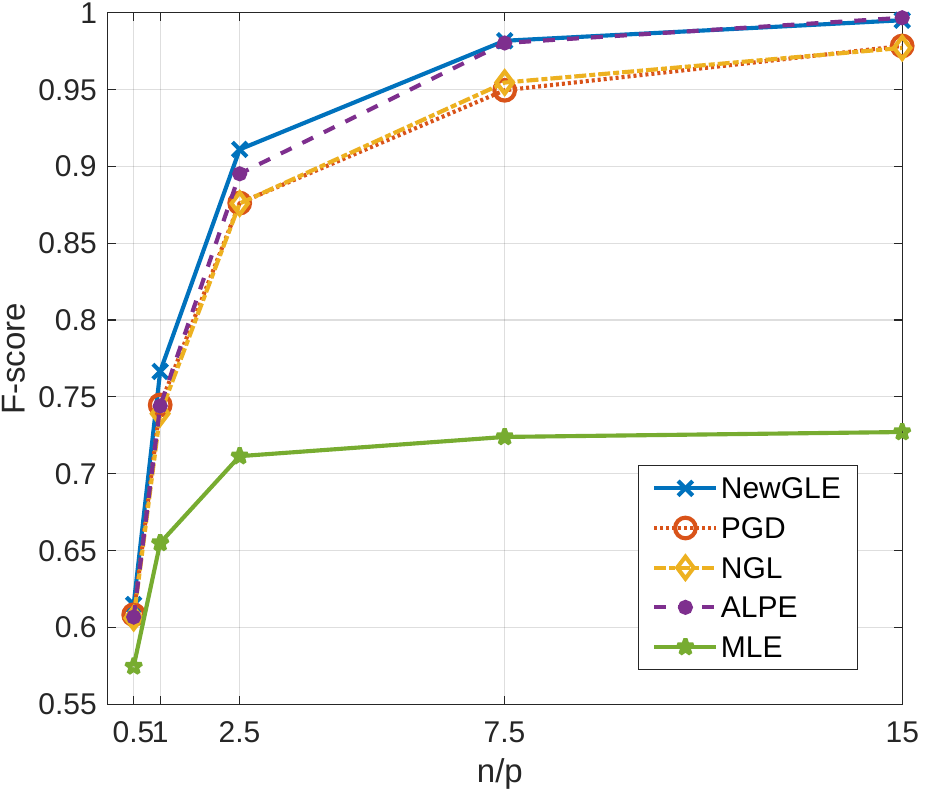} }}
    \caption{Performance comparisons in terms of (a) RE,  and (b) F-score with different sample size ratios $n/p$ for
    Barabasi-Albert graph of degree 2 with 100 nodes.}
    \label{ RE and vs time p = 100}
\end{figure}

Figures \ref{fig: RE and vs time p = 1000} and \ref{ RE and vs time p = 100} show the graph learning performance of different methods for learning the topology of planar graphs with 1,000 nodes (Fig. \ref{fig: RE and vs time p = 1000}) and of  Barabasi-Albert graphs of degree two with 100 nodes (Fig. \ref{ RE and vs time p = 100}). The performance is presented in terms of RE and F-score versus the sample size ratio $n/p$, where $n$ is the number of data samples used to calculate the sample covariance matrix, $\Smat$, and $p$ is the number of nodes.

Figures \ref{fig: RE and vs time p = 1000}(a) and \ref{ RE and vs time p = 100}(a) show that the REs of all the estimators decrease as the sample size ratio, $n/p$, increases.
Figures \ref{fig: RE and vs time p = 1000}(b) and \ref{ RE and vs time p = 100}(b) show that the F-score of all the methods increases as the sample size ratio $n/p$ increases. Figures  \ref{fig: RE and vs time p = 1000} and \ref{ RE and vs time p = 100} show that the proposed NewGLE method outperforms all compared methods (ALPE, NGL, PGD, and MLE) in both the RE and F-score senses. 
It should be noted that in \ref{fig: RE and vs time p = 1000}(b) the comparison with the MLE has been removed since it has a significantly lower F-score than the other methods and in order to highlight the differences between the other methods.
As shown in Figs. \ref{fig: RE and vs time p = 1000}(b) and \ref{ RE and vs time p = 100}(b) for a small sample size ratio, the methods do not provide a perfect F-score. This is because there may not be enough data samples to recover the graph connectivity effectively.
Still, the proposed NewGLE method consistently outperforms all compared methods in the F-score sense, especially when the sample size ratio is small. This outcome is significant in large-scale problems since, in these cases, the sample size ratio is usually small.
\subsection{Results - Computational Costs}
\label{ResultsCC_subsection}
\begin{figure}[hbt]
    \centering
    \subfloat[\centering]{{\includegraphics[width=0.25\textwidth]{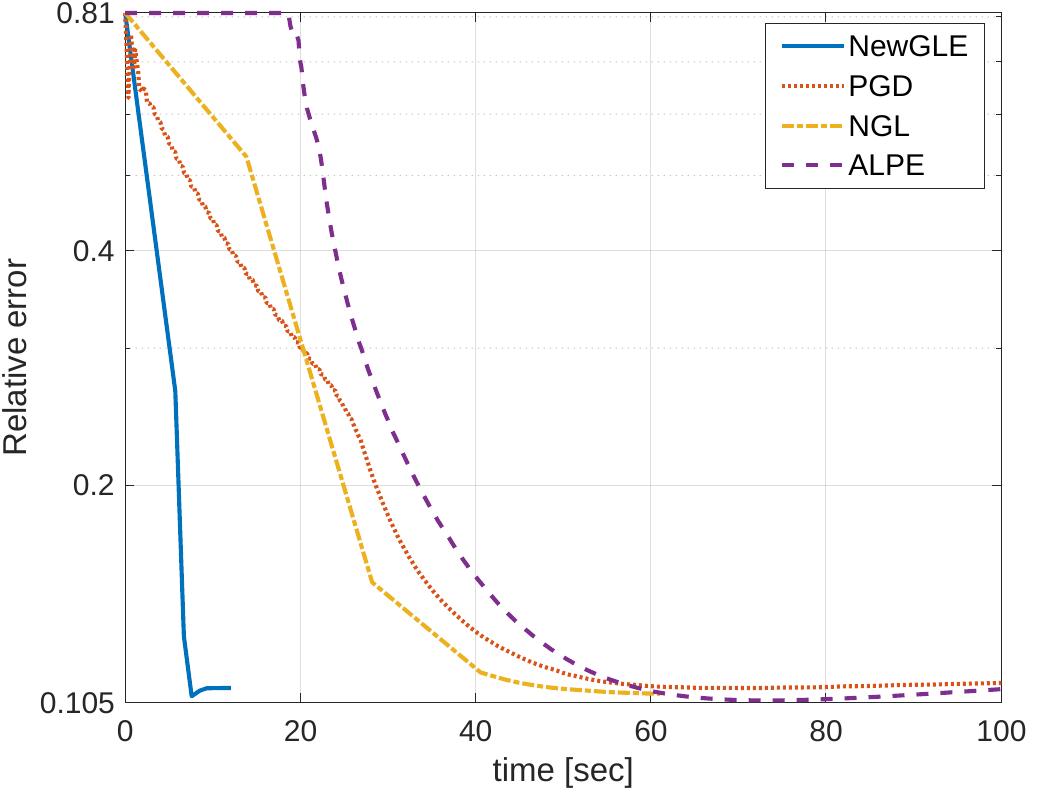} }}%
    \subfloat[\centering]{{\includegraphics[width=0.245\textwidth]{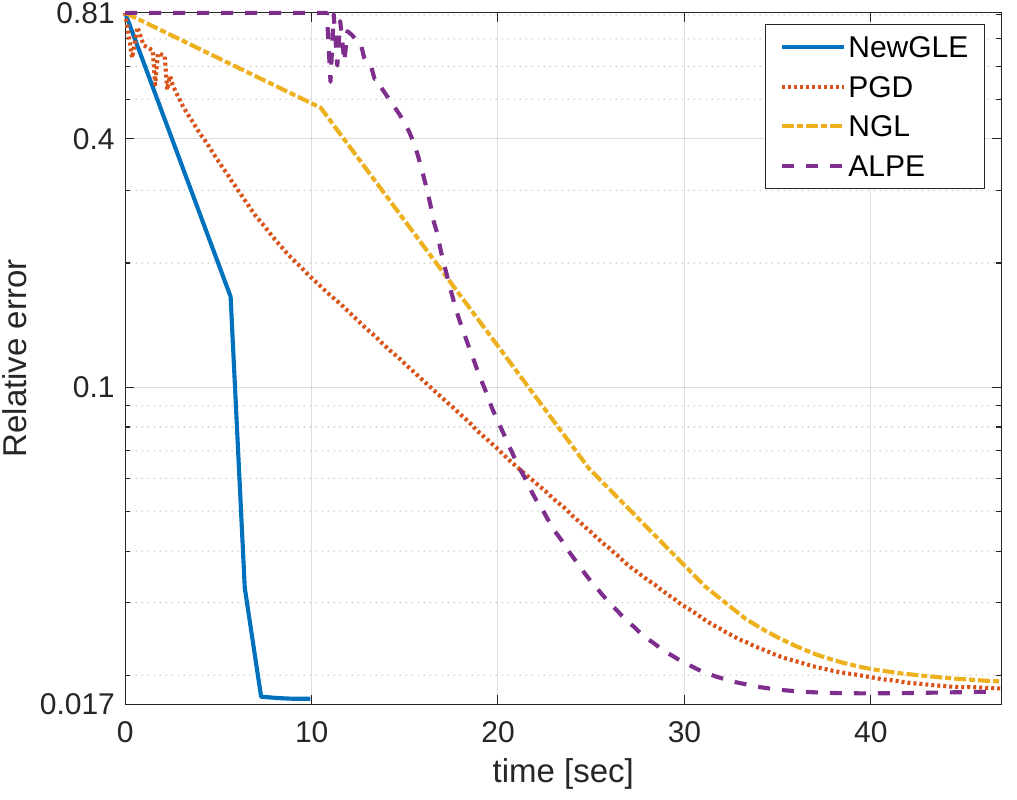} }}
    \caption{Run-time performance comparisons under RE with a sample size ratio (a) $n/p = 0.5$ and (b) $n/p = 15$ for random planar graphs with 1,000 nodes.}
    \label{fig: RE vs time}
\end{figure}

\begin{figure}[hbt]
    \centering
    \subfloat[\centering]{{\includegraphics[width=0.245\textwidth]{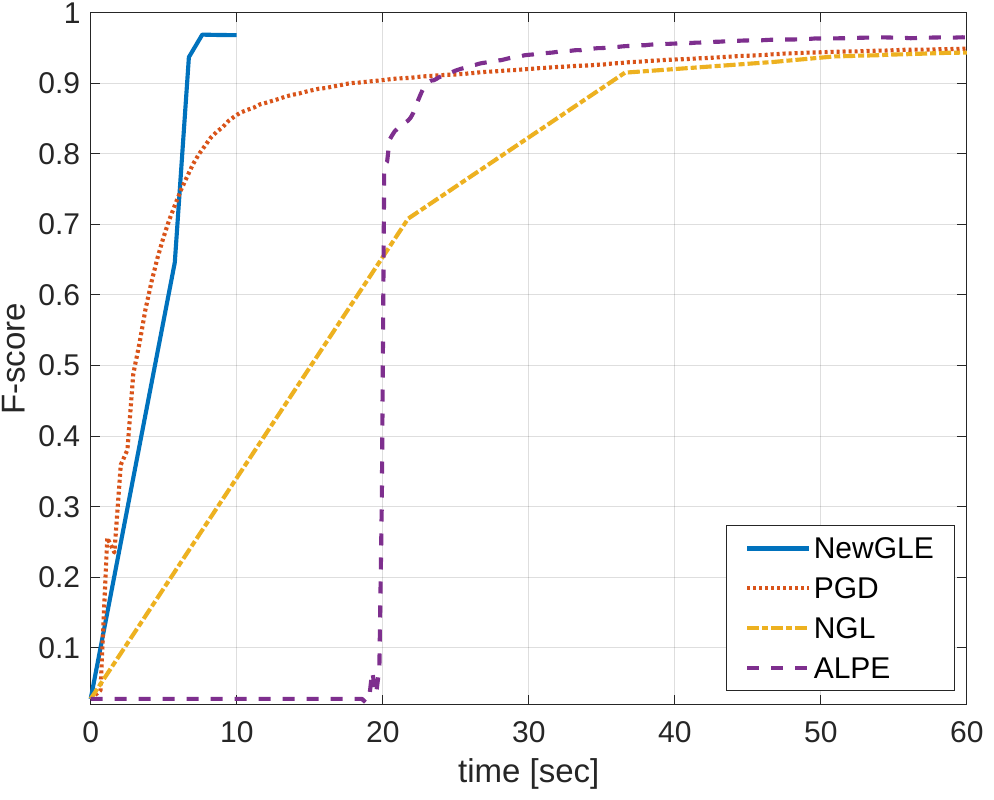} }}
    \subfloat[\centering]{{\includegraphics[width=0.245\textwidth]{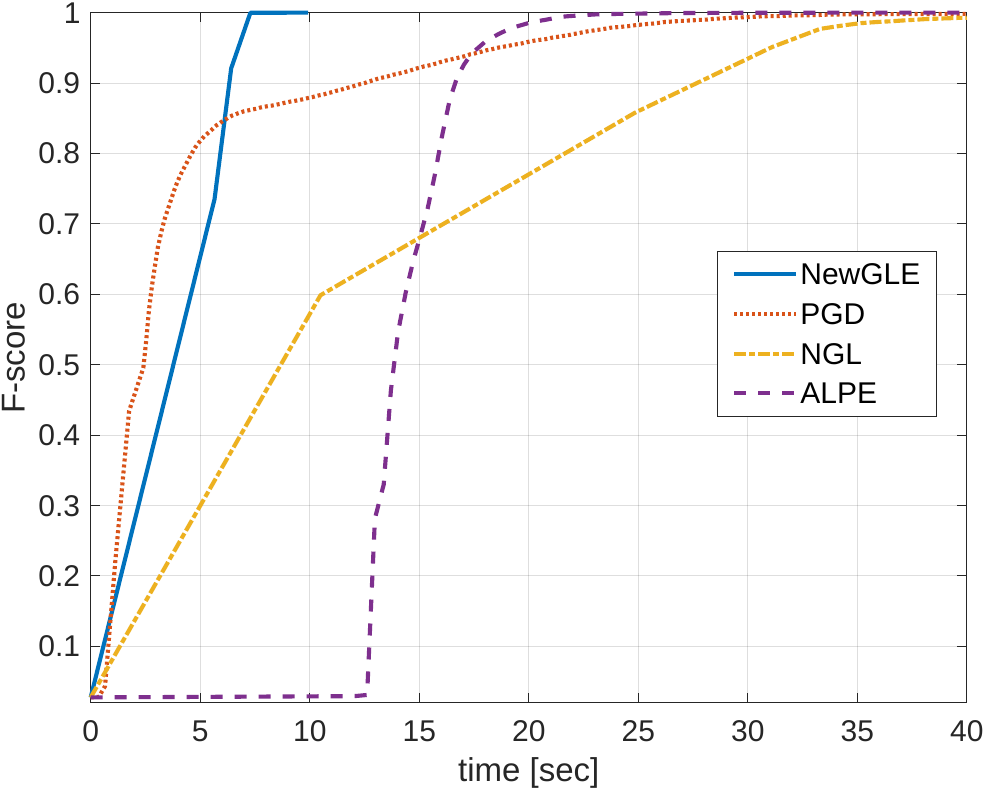} }}
    \caption{Run-time performance comparisons under F-score with a sample size ratio (a) $n/p = 0.5$ and (b) $n/p = 15$ for random planar graphs with 1,000 nodes.}
    \label{fig: FS vs time}
\end{figure}

\begin{figure}
    \centering
{{\includegraphics[width=0.4\textwidth]{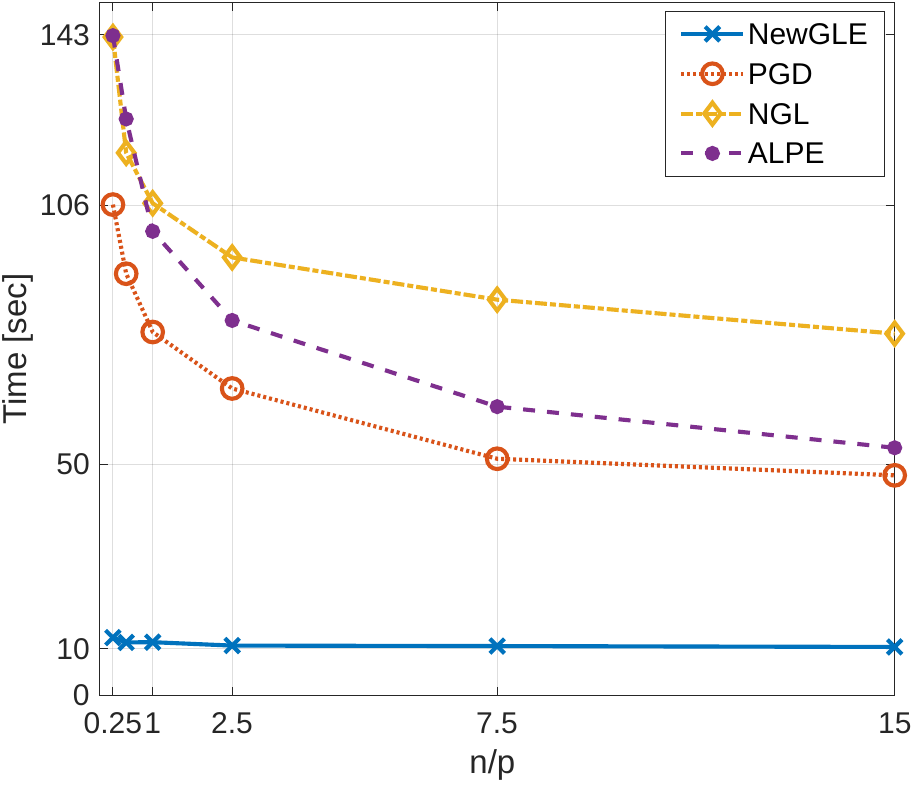} }}%
    \caption{Average convergence times over sample size ratio $n/p$ for random planar graphs with 1,000 nodes.}
    \label{fig: time vs nDp}
\end{figure}

Figures \ref{fig: RE vs time}-\ref{fig: time vs nDp} compare the computational efficiency of the proposed NewGLE method with the other methods in terms of run time. These experiments were conducted on relatively large planar graphs of $p=1,000$ nodes so that the measuring times less depend on minor implementation details.
Figure \ref{fig: RE vs time} shows the convergence of the compared methods in terms of RE for two sample size ratios: (a) $n/p = 0.5$, and (b) $n/p = 15$. Notably, the proposed method converges significantly faster than the other methods in both cases, while the difference between the methods is more significant for the case when the sample size ratio is smaller, $n/p = 0.5$.

Figure \ref{fig: FS vs time} shows the F-score of the compared methods during two realizations: (a) with sample size ratio $n/p = 0.5$, and (b) with sample size ratio $n/p = 15$. Notably, the proposed method reached the highest F-score and got it significantly faster than the run time taken for the other methods to achieve their highest F-score values. In addition, it can be seen that in the case of $n/p =15$, all the methods succeeded in estimating the connectivity of the graph, and the NewGLE, PGD, and ALPE methods managed to estimate it perfectly (i.e., F-score = 1). In addition, it can be seen that the advantage of the proposed method is more notable for a smaller sample size ratio, i.e., $n/p = 0.5$.
Additionally, the results presented in Figs. \ref{fig: RE vs time} and \ref{fig: FS vs time} reveal that while the RE and F-score performance of the ALPE method, as demonstrated in Subsection \eqref{ResultsGL_subsection}, is commendable and comparable to the proposed method, the implementation of the ALPE method necessitates a significant investment in computational resources to effectively learn a graph, as measured by RE and F-score metrics. The reason for this is that the ALPE method requires an estimation of the weights, which are used for sparsity-promoting regularization.

Finally, Fig. \ref{fig: time vs nDp} presents the average convergence time for each method over the sample size ratio $n/p$. It shows that the average convergence time of all estimators decreases as the sample size ratio $n/p$ increases. In addition, it emphasizes the significant advantage of the proposed method in terms of computational efficiency, especially when there are few samples and the sample size ratio is small (i.e., $n/p \leq 1$), which is typically the case in large-scale problems.
More experiments for real data are presented in Fig. \ref{fig:bio} in the Appendix, demonstrating the computational efficiency of NewGLE.

\subsection{Ablation Study}
In this section, we provide some intuition on different features of our method. Specifically, we demonstrate the influence of the following three choices:
\begin{enumerate}
\item The free set mechanism.
\item The nonlinear conjugate gradients method for the constrained Newton problem.
\item The influence of the $\varepsilon$ parameters.
\item The influence of the diagonal preconditioner.
\end{enumerate}

We show the performance using the random planar graph example with $p=1,000$ nodes and $n=500$ samples. Each run includes the NewGLE method without one of the algorithmic features. The regularization parameter for the experiment is $\lambda=0.25$, which is the best performing choice in terms of the error and F-score for this case. Figure \ref{fig:Ablation} shows a typical instance of this example, demonstrating the behavior of NewGLE without each of the features. All the results are comparable in their accuracy (error and F-scores). It is clear that the most important feature in our study is the usage of nonlinear projected conjugate gradients for solving the inner Newton problem. Without it (i.e., using preconditioned gradient descent), the method is much slower. The second most important feature is the diagonal preconditioner, which approximately halves the cost in this example, and is significant also in other cases in our experience. The use of the free set imposes only a minor improvement in this example and is the less dominant feature. The use of $\varepsilon$ slows down the algorithm a bit, but it is important for theoretical stability, according to our theory in Theorem \ref{Thm2}.

\begin{figure}
    \centering
{{\includegraphics[width=0.4\textwidth]{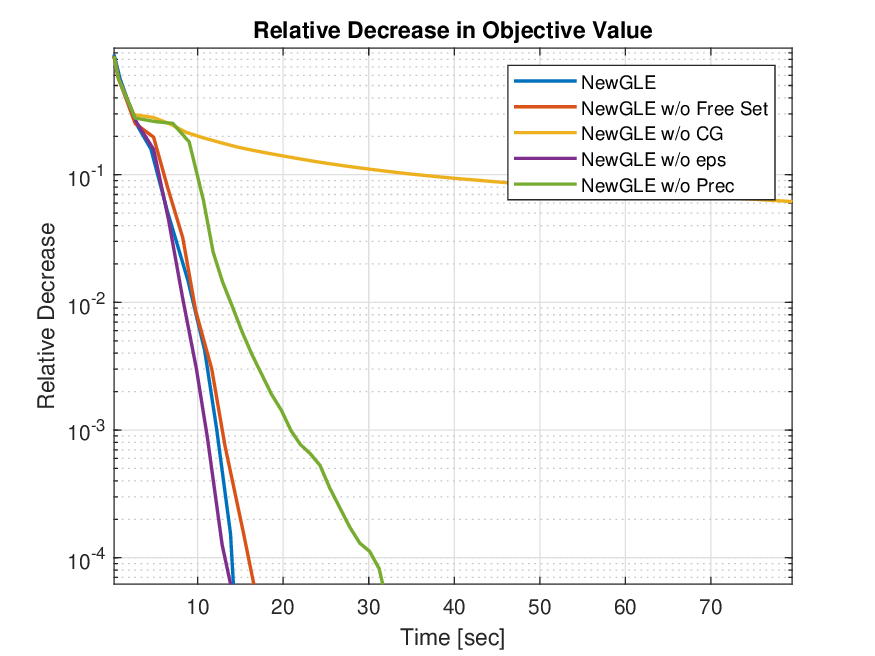} }}%
    \caption{Ablation study of computational times versus algorithmic choices.}
    \label{fig:Ablation}
\end{figure}

\section{CONCLUSION}
\label{conc}

In this paper, we developed a proximal Newton method for minimizing a regularized maximum likelihood objective function for estimating a sparse graph Laplacian precision matrix from data distributed according to an LGMRF model. Our method is based on a second-order approximation of the smooth part of the objective function, while keeping the regularization and constraints intact. Methods of this type have been proven to be efficient for estimating standard GMRF models using $\ell_1$-norm penalty, while in our case, the estimation is also subject to Laplacian constraint. 
Furthermore, we used the nonconvex MCP function to promote the sparsity of the resulting graph without introducing a significant bias. Beyond the quadratic approximation itself, we used three algorithmic features to treat the Newton problem at each iteration efficiently: 1) we used a free/active set splitting, (2) we used a nonlinear conjugate gradient method, and (3) we used diagonal preconditioning. 
This resulted in an accurate and efficient estimation of the sparse weighted graph.

Numerical studies demonstrated the effectiveness of the proposed algorithm in terms of accuracy and computational complexity. The proposed method outperformed recent methods in terms of convergence time, while achieving similar or better F-scores and relative errors, particularly for problems with a small sample size ratio (i.e., $n/p < 1$). The plausible performance of the proposed method is particularly important in high-dimensional problems, where computational efficiency is crucial, and sample scarcity is a common issue.

\section*{ACKNOWLEDGEMENTS}
This research was supported by the Israel Ministry of National Infrastructure, Energy, and Water Resources and by the ISRAEL SCIENCE FOUNDATION (grant No. 1148/22). This work was also supported in part by the Israeli Council for Higher Education (CHE) via the Data Science Research Center, Ben-Gurion University of the Negev, Israel.

\bibliography{paper_AISTATS2024_CameraReady_476}

\section*{Checklist}

 \begin{enumerate}
 \item For all models and algorithms presented, check if you include:
 \begin{enumerate}
   \item A clear description of the mathematical setting, assumptions, algorithm, and/or model. Yes.
   \item An analysis of the properties and complexity (time, space, sample size) of any algorithm. We provide an analysis of the convergence guarantees.
  \item (Optional) 
  Anonymized source code, with specification of all dependencies, including external libraries. Yes.
 \end{enumerate}

 \item For any theoretical claim, check if you include:
 \begin{enumerate}
   \item Statements of the full set of assumptions of all theoretical results. Yes.
   \item Complete proofs of all theoretical results. Yes.
   \item Clear explanations of any assumptions. Yes.
 \end{enumerate}

 \item For all figures and tables that present empirical results, check if you include:
 \begin{enumerate}
   \item The code, data, and instructions needed to reproduce the main experimental results (either in the supplemental material or as a URL). Yes. Instructions to reproduce the experiments are in the paper. Our MATLAB code is publicly available on Github. 
   \item All the training details (e.g., data splits, hyperparameters, how they were chosen). Hyperparameters are given in the paper, including instructions on the data generation. 
         \item A clear definition of the specific measure or statistics and error bars (e.g., with respect to the random seed after running experiments multiple times). Yes, although the variance is quite low and hence omitted.
         \item A description of the computing infrastructure used. (e.g., type of GPUs, internal cluster, or cloud provider). Yes.
 \end{enumerate}

 \item If you are using existing assets (e.g., code, data, models) or curating/releasing new assets, check if you include:
 \begin{enumerate}
   \item Citations of the creator If your work uses existing assets. Yes.
   \item The license information of the assets, if applicable. Not Applicable.
   \item New assets either in the supplemental material or as a URL, if applicable. Not Applicable.
   \item Information about consent from data providers/curators. Not Applicable.
   \item Discussion of sensible content if applicable, e.g., personally identifiable information or offensive content. Not Applicable.
 \end{enumerate}

 \item If you used crowdsourcing or conducted research with human subjects, check if you include:
 \begin{enumerate}
   \item The full text of instructions given to participants and screenshots. Not Applicable.
   \item Descriptions of potential participant risks, with links to Institutional Review Board (IRB) approvals if applicable. Not Applicable.
   \item The estimated hourly wage paid to participants and the total amount spent on participant compensation. Not Applicable.
 \end{enumerate}

 \end{enumerate}

\newpage

\onecolumn
\aistatstitle{Efficient Graph Laplacian Estimation by Proximal Newton: \\ Supplementary Materials}

\section{PROOF OF THEOREM \ref{Th1}}
\label{sec:ProofThm1}

In this section we prove Theorem \ref{Th1} (which is also repeated below). We note that this property was used in the works \citep{egilmez2017graph,ying2020nonconvex} without explicit proof, and here we provide it. 

\textbf{Theorem 1}\textit{
The optimization problem 
\begin{equation}\nonumber
\begin{aligned}
    \underset{\Lmat\in \mathcal{F}}{\text{minimize}}
     \quad & \TR{\Lmat \Smat} -\log|\Lmat+\Jmat|
        +\rho\braro{\Lmat;\lambda},
\end{aligned}
\end{equation}
where the feasible set is given by  
\beqna \nonumber
\mathcal{F}=\left\{
    \Lmat \in \mathcal{S}^p \Big{|} \begin{array}{ll} 
    \Lmat_{i,j}\leq 0, \forall i \neq j,~i,j=1,\ldots,p,\;\; \Lmat\onevec=\zerovec
    \end{array}\right\}.
\eeqna
is equivalent to the 
optimization problem stated in \eqref{minimization_problem_3}}.

\begin{proof} 
	
Since the objective functions in \eqref{minimization_problem_3} and in \eqref{minimization problem 4} are identical (through \eqref{pseudoL}), it remains to prove that the feasible sets (solution spaces) of the two problems are equal under this objective function. It should be noted that there are two distinctions between the feasible sets $\mathcal{L}$ and $\mathcal{F}$ in problems \eqref{minimization_problem_3} and \eqref{minimization problem 4}, respectively. The first is that $\mathcal{L}$ includes a positive semi-definite (PSD) constraint (i.e., $ \Lmat\in \mathcal{S}^p_+$), and the second is that $\mathcal{L}$ includes a rank constraint (i.e., $\text{rank}\braro{\Lmat} = p-1$).

To address the PSD constraint in problem \eqref{minimization_problem_3}, we observe that the constraints $\Lmat\onevec=\zerovec,$ and $\brasq{\Lmat}_{i,j}\leq 0$ for $i \neq j$ are present. Therefore, it follows that every matrix $\Lmat\in\mathcal{F}$ is a diagonally dominant matrix with positive diagonal entries. Hence, requiring symmetry is sufficient to ensure that the minimizer of \eqref{minimization problem 4} is PSD as well \citep{horn2012matrix}.

Next, we show that for any graph Laplacian $\Lmat\in\mathcal{F}$, adding the constraint ${\textrm{rank}}\braro{\Lmat} = p-1$ is either equivalent to ${\textrm{rank}}\braro{\Lmat+\Jmat} = p$, or that the matrix $\braro{\Lmat+\Jmat}$ is non-singular under the considered objective function. To this end, we note that any $\Lmat\in\mathcal{F}$ admits an eigenvalue decomposition $\Umat\Lambda_L\Umat^T$, where $\Umat$ is $p \times p$ matrix of the eigenvectors of $\Lmat$, and
$\Lambda_L$ is a diagonal eigenvalue matrix as follows
    \be
        \Lambda_L = \textrm{diag}\braro{\brasq{0,\lambda_2,\dots,\lambda_p}^T},
    \ee
where $0,\lambda_2,\ldots,\lambda_p$ are the eigenvalues of $\Lmat$. The rank of $\Lmat$ is the number of positive eigenvalues. In addition, the first eigenvector of $\Lmat$ is  $\frac{1}{\sqrt{p}}{\mathbf{1}}$ and it corresponds to the eigenvalue $\lambda_1=0$. On the other hand, an eigendecomposition of $\Jmat$ can be $\Umat\Lambda_J\Umat^T$, where 
\be
     \Lambda_J = \textrm{diag}\braro{\brasq{1,0,\dots,0}^T}
\ee
and the eigenvector of $\Jmat$ that is associated with the single nonzero eigenvalue is $\frac{1}{\sqrt{p}}{\mathbf{1}}$---the same one as in $\Lmat$. Hence, for the same eigenvector matrix $\Umat$, the eigenvalues of the matrix $\Lmat+\Jmat$ are $1,\lambda_2,\ldots,\lambda_p$, and ${\textrm{rank}}\braro{\Lmat+\Jmat} = {\textrm{rank}}\braro{\Lmat}+1$. Hence, to conclude, if $\Lmat\in\mathcal{F}$ and  $\braro{\Lmat +\Jmat}$ is a non-singular matrix, then ${\textrm{rank}}\braro{\Lmat} = p-1$.


Now we come to the last part of the proof. Assume by contradiction that the minimizer of \eqref{minimization problem 4} satisfies that the matrix $\braro{\Lmat+\Jmat}$ is singular, then the expression $-\log{\braro{|\Lmat+\Jmat|}}$ in the objective approaches infinity. Thus, any other matrix $\Lmat\in \mathcal{F}$ for which the matrix $\braro{\Lmat+\Jmat}$ is a non-singular matrix will yield a lower (finite) objective value in \eqref{minimization problem 4}. Hence it can be concluded that the solution of \eqref{minimization problem 4} satisfies that the $\braro{\Lmat+\Jmat}$ matrix is non-singular, i.e., ${\textrm{rank}}\braro{\Lmat+\Jmat} = p$ and hence ${\textrm{rank}}\braro{\Lmat} = p-1$.  
\end{proof}

\vfill 
\section{DERIVATION OF THE PARAMETERIZED NEWTON PROBLEM} \label{sec:paramNewton}

First, we simplify the writing constraints of \eqref{Newton Problem}. We assume that the previous iterate $\Lmat^{(t)}$ satisfies the constraints, and by definition $\Lmat^{(t+1)}$ will satisfy them as well. To achieve that we also choose $\Lmat^{(0)}\in\mathcal{F}$. Based on the definition of $\mathcal{F}$ in \eqref{Lc - 9}, under the assumption that $\Lmat^{(t)}$ satisfies the constraints, in order to obtain an update $\Lmat^{(t+1)}$ that  satisfy these constraints, we  require the following conditions on the variable  $\Delta$:
\begin{itemize}
    \item $\Lmat^{(t)}+\Delta \in \mathcal{S}^p$. Hence, we require that $\Delta \in \mathcal{S}^p$.
    \item The off-diagonal elements should satisfy  $\brasq{\Lmat^{(t)}+\Delta}_{i,j}\leq 0$,  $\forall i\neq j$. Hence, we require that  $\Delta_{i,j}\leq-\Lmat^{(t)}_{i,j}$.
    \item $\braro{\Lmat^{(t)}+\Delta}\onevec = \zerovec$. Assuming that  $\Lmat^{(t)}\onevec = \zerovec$, then we only need to require $\Delta\onevec = \zerovec$.  
\end{itemize}
By substituting these conditions in \eqref{Newton Problem} and removing the constant term in the objective,  we can rewrite the Newton problem from \eqref{Newton Problem} as follows:
\begin{equation}
\begin{aligned}
\label{Newton Problem 1.5}
    \underset{\Delta \in \mathcal{S}^p}{\text{minimize}}
     \quad &  \TR{\Delta\braro{\Smat - \Qmat}} + \frac{1}{2}\TR{\Delta\Qmat\Delta\Qmat} + \frac{1}{2}\|\varepsilon\odot\Delta\|_F^2 +\rho_{MCP}\braro{\Lmat^{(t)}+\Delta;\lambda}\\
    \textrm{s.t.}
    \quad & \Delta_{i,j}\leq-\Lmat^{(t)}_{i,j},  \forall i\neq j \\
    \quad & \Delta\onevec = \zerovec.
\end{aligned}
\end{equation}

Next, we further simplify the problem by replacing the symmetric unknown matrix $\Delta\in\mathcal{S}^p$ with a vector $\delta\in\mathbb{R}^{p(p-1)/2}$ by using Definition \ref{DEF2} of the linear operator $\mathcal{P}: \mathbb{R} ^{p(p-1)/2}$. According to this definition, given any $\wvec \geq 0$, $\Lx{\wvec}$ satisfies $\sum_j\Lx{\wvec}_{i,j} = 0 $ for any $i \in \bracu{1,\dots,p}$, and $\Lx{\wvec}_{i,j} = \Lx{\wvec}_{j,i}$ for any $i \neq j$. In other words, $\Lx{\wvec}$ is a Laplacian matrix for any $\wvec$ which has nonnegative values.  The simple example below illustrates the definition of $\mathcal{P}$. Let $\wvec=[w_1,w_2,w_3] \in \mathbb{R}^3$ be a nonnegative vector. Then
\beqna
\label{eq:P_example}
\Lx{\wvec} = 
\begin{bmatrix}
\sum_{i=1,2}w_i & -w_1 & -w_2\\
-w_1 & \sum_{i=1,3}w_i & -w_3\\
-w_2 & -w_3 & \sum_{i=2,3}w_i\\
\end{bmatrix} .
\eeqna

Since the Laplacian is a symmetric matrix, we can rewrite the sparse promoting term in \eqref{Newton Problem 1.5} as:
\begin{equation}
\begin{aligned}
    \label{r mcp for vec}
\rho_{MCP}\braro{\Lx{\wvec};\lambda} &\overset{(a)}{=}  \sum_{i\neq j}MCP\braro{\Lmat_{ij};\gamma,\lambda} 
\\ &\overset{(b)}{=}
 2\sum_{i> j}MCP\braro{\Lmat_{ij};\gamma,\lambda}\\ &\overset{(c)}{=} 
2\sum_{k=1}^{p\braro{p-1}/2}MCP\braro{-w_{k};\gamma,\lambda} 
\\ &\overset{(d)}{=}2\tilde{\rho}_{mcp}\braro{{\wvec};\lambda},
\end{aligned}
\end{equation}
where $\Lmat = \Lx{\wvec}$. 
In \eqref{r mcp for vec}, 
$(a)$ is obtained by substituting \eqref{r mcp}, $(b)$ is obtained from the symmetry of $\Lmat$ (i.e., since $\Lmat\in \mathcal{S}^p$), and $(c)$ is obtained by substituting \eqref{Lx matrix}. Finally,  $(d)$ is obtained since $\tilde{\rho}_{MCP}\braro{{\wvec};\lambda}$ defined in \eqref{rho_def} is an even function. Noting that for $\wvec,\delta\in\mathbb{R}^{p(p-1)/2}$, such that $\Lmat = \Lx{\wvec}$ and $\Delta = \Lx{\delta}$, we can write:
\begin{itemize}
    \item  $\rho_{mcp}\braro{\Lmat+\Delta;\lambda} \overset{\eqref{r mcp for vec}}{=}2\tilde{\rho}_{mcp}\braro{\wvec+\delta;\lambda}$.
    \item $\Delta \in \bracu{\Delta\in\mathcal{S}^p| \Delta\onevec = \zerovec,\Delta_{i,j}\leq-\Lmat^{(k)}_{i,j},  \forall i\neq j}$
    $\overset{\eqref{Lx matrix}}{\iff} \delta \in \bracu{\delta\in \mathbb{R}^{p(p-1)/2}|\quad \delta \geq -\wvec}$.
    \item $\frac{1}{2}\|\varepsilon\odot \Delta\|_F^2 = \|\tilde\varepsilon\odot\delta\|_2^2$, which holds trivially because the diagonal of $\varepsilon$ is 0 and it is symmetric. 
\end{itemize} 
Hence, \eqref{Newton Problem 1.5}, which is equivalent to \eqref{Newton Problem}, is also equivalent to \eqref{Newton Problem 2}, which is the problem we solve in practice in the inner Newton problem.

\section{NONLINEAR PRECONDITIONED CONJUGATE GRADIENTS}\label{sec:NLCG}
Generally, iterates of projection methods are obtained by
\be
\label{CG}
\delta^{(t+1)} = \Pi_U\bracu{\delta^{(t)} + \zeta\dvec^{(t)}},
\ee
where $\delta^{(t)}$ is the current iterate, $\dvec^{(t)}$ is the descent (search) direction, $\zeta>0$ is a step-size parameter, $U$ is some feasible set, and $\Pi_U\bracu{\vvec}$ is a projection of $\vvec$ onto $U$. That is, using the $\ell_2$ norm we have $\Pi\{\vvec\}=\arg\min_\xvec\|\xvec-\vvec\|_2$. Since there is only a nonnegativity constraint in \eqref{Newton Problem 2}, the projection here is simple to calculate and is given by:
\be
\label{proj}
    \Pi_{\vvec+\wvec \geq \zerovec} \bracu{\vvec} = \textrm{max}\bracu{\vvec,-\wvec}.
\ee
The step-size $\zeta>0$ is computed by a linesearch procedure, and the descent direction $\dvec^{(t)}$ is
given by
\beqna
    \label{d CG}
    \dvec^{(t)}=
    \left\{\begin{array}{lr}
    -\nabla_{\delta} f_N \braro{\delta^{(t)}} &  \textrm{if} \quad t = 0\\
    -\nabla_{\delta} f_N \braro{\delta^{(t)}} + \beta_t\dvec^{(t-1)}  & \textrm{if} \quad t>0 
    \end{array},\right.
\eeqna
where $\beta_t$ is a scalar and ${f_N(\delta^{(t)})}$ is the  
objective defined in \eqref{Newton Problem 2} at the point $\delta^{(t)}$. Since 1952, there have been many formulas for the scalar $\beta_t$, and in this paper, we use the DY method \citep{dai1999nonlinear}, given by
\be
\label{beta DY}
    \beta^{DY}_t = \textrm{max}\bracu{0,\frac{\|\nabla_{\delta} f^{(t)}_N\|^2_2}{\braro{\nabla_{\delta} f^{(t)}_N - \nabla_\delta f^{(t-1)}_N}^T\dvec^{(t-1)}}}.
\ee
One remarkable property of the DY method is that it produces a descent direction at each iteration and converges globally for convex problems in the case where linesearch is used \citep{dai1999nonlinear}.

To apply the projected NLCG, the gradient of $f_N(\delta)$ (defined in \eqref{Newton Problem 2}) is needed and is given by
\beqna
\label{grad by delta}
\begin{aligned}
    \nabla_{\delta}f_N = \mathcal{P}^*\braro{\nabla_\Delta f_N}= \quad &
    \mathcal{P}^*\braro{\Smat - \Qmat + \Qmat\Lx{\delta}\Qmat} + 2\tilde\varepsilon^2\odot\delta + 2\nabla_{\delta}\tilde{\rho}_{mcp}\braro{\wvec+\delta;\lambda},
\end{aligned}
\eeqna
where $\Lx{\delta} = \Delta$, and $\mathcal{P}^*$ is the adjoint of $\mathcal{P}$, that is
\begin{equation}\label{eq:adjoint P}
\brasq{\mathcal{P}^*\braro{\Ymat}}_k \define Y_{ii} + Y_{jj} - Y_{ij} - Y_{ji}
\end{equation}
for $k$ as in Definition \ref{DEF2}. Here we see the advantage of the proximal Newton approach. In \eqref{grad by delta} the gradient of $\tilde\rho$ is computed at $\wvec+\delta$, instead of $\wvec$ alone if we were to naively majorize $\tilde\rho$ in the outer Newton problem.  So, the inner solver is able to adapt to the curvature of $\tilde \rho$ and to the different sections in its definition.

\subsection{Diagonal Preconditioning}
\label{preconditioner_subsection}
The NLCG method that we use in this work is essentially an accelerated gradient-based method. To further accelerate it with rather minimal computational effort, a diagonal approximation to the Hessian can be used, acting as a preconditioner in the NLCG method. The aim of the diagonal preconditioner is
to properly scale the gradient and speed up the rate of convergence of the iterative inner method for the Newton direction. In this paper, we suggest using the diagonal of the Hessian with respect to $\delta$ as a preconditioner with NLCG. Below we show the derivation of this preconditioner. 

Our diagonal preconditioner matrix, $\Dmat$, is defined as the diagonal of the Hessian of $f_N$. To ease the derivation, we ignore the $\varepsilon$ term of \eqref{Newton Problem 2} in this section. Following the Hessian of $f$ at \eqref{grad and h} and considering the linear operator $\mathcal{P}$ as a matrix in $\mathbb{R}^{p^2\times p(p-1)/2}$ (i.e., ignoring the vector-to-matrix reshaping in its definition), the Hessian of $f_N$ is
\begin{equation}
\Hmat_N = \mathcal{P}^\top\left(\Qmat\otimes\Qmat\right)\mathcal{P}\in\mathbb{R}^{p(p-1)/2\times p(p-1)/2}.
\end{equation}
To get its $(k,k)$ diagonal entry for our preconditioner, we multiply $\Hmat_N$ by a unit vector from both sides:
\begin{equation}\label{eq:Dkk}
[\Dmat]_{k,k} = [\Hmat_N]_{k,k} = \mathbf{e}_k^T\Hmat_N \mathbf{e}_k,
\end{equation}
where $\mathbf{e}_k\in\mathbb{R}^{p(p-1)/2}$ is the unit vector of zeros, except the $k$-th entry, which equals 1. Equivalently to \eqref{eq:Dkk}, we can write 
\begin{equation}
\label{eq:diagonal_of_H}
[\Hmat_N]_{k,k} = \TR{\mathcal{P}(\mathbf{e}_k)\Qmat\mathcal{P}(\mathbf{e}_k)\Qmat},
\end{equation}
where the matrix $\mathcal{P}(\mathbf{e}_k)\in\mathbb{R}^{p\times p}$ is a matrix that is essentially the reshape of the $k$-th column of the operator $\mathcal{P}$, if we consider it as a matrix. If we look at the example of \eqref{eq:P_example}, then taking $k=1$ corresponds to
$$
\mathcal{P}(\mathbf{e}_1) = \begin{bmatrix}
\;\;1 & -1 & 0\\
-1 & \;\;1 & 0\\
\;\;0 &\;\;0 & 0\\
\end{bmatrix}. 
$$ 
Note that $k=1$ in \eqref{eq:P_example} influences only the $2\times2$ sub-matrix of the indices $(1,2)$ in $\mathcal{P}(\mathbf{e}_1)$. More generally, $\mathcal{P}(\mathbf{e}_k)$ results in a matrix of zeros, except a $2\times2$ nonzero submatrix of the indices $(i,j)$  where $k$ corresponds to the pair $(i,j)$ according to Definition \ref{DEF2}. Hence, multiplying $\mathcal{P}(\mathbf{e}_k)\Qmat$, will result in a matrix of zeros except the $i$-th and $j$-th \emph{rows}, which are equal to $\textbf{q}_i - \textbf{q}_j$ and $\textbf{q}_j - \textbf{q}_i$, respectively, where $\textbf{q}_i$ is the $i$-th column/row of (the symmetric) $\Qmat$. According to \eqref{eq:diagonal_of_H} we now have 
\begin{eqnarray}
\label{Hkk 3}
[\Hmat_N]_{k,k} &=& \TR{(\mathcal{P}(\mathbf{e}_k)\Qmat)^2} \nonumber \\
&=& \TR{\begin{pmatrix}
[\Qmat]_{i,i} - [\Qmat]_{j,i} & [\Qmat]_{i,j} - [\Qmat]_{j,j}\\ [\Qmat]_{j,i} - [\Qmat]_{i,i} & [\Qmat]_{j,j} - [\Qmat]_{j,i}
\end{pmatrix}^2} \nonumber \\
& = & \braro{\brasq{\Qmat}_{i,i}+\brasq{\Qmat}_{j,j}-2\brasq{\Qmat}_{i,j}}^2 = [\Dmat]_{k,k},
\end{eqnarray}
where $\braro{i,j}$ corresponds to the index $k$  according to Definition \ref{DEF2}. In the last equality we used the symmetry of $\Qmat$, i.e., $[\Qmat]_{j,i}=[\Qmat]_{i,j}$. Below, Algorithm \ref{algorithm for Newton problem} summarizes the projected and preconditioned NLCG method for solving \eqref{Newton Problem}. Note that the linesearch in the last bullet of the algorithm is applied before the projection, hence taking a smaller steplength $\xi$ means we effectively enlarge the Hessian approximation at each step.

\begin{algorithm}[hbt]
\caption{Projected \& preconditioned NLCG}
\label{algorithm for Newton problem}
    \SetAlgoLined
    \KwInput{$\Lmat^{(t)} = \Lx{\wvec},\Qmat,\text{Free}\braro{\Lmat^{(t)}},\lambda$}
    \KwResult{Newton direction $\Delta^{(t)}$}
    \KwInit{  Set $\delta^{(0)} =\zerovec$}
    
    \While{Stopping criterion is not achieved}{
         \begin{itemize}
            \item Compute the preconditioner $\Dmat$ as in \eqref{Hkk 3}.\\
            \item Compute the gradient  $\nabla_{\delta}f_N$ as in \eqref{grad by delta}.\\
            \item Apply the preconditioner $\gvec = 
            \braro{\Dmat^{-1}\cdot\nabla_{\delta}f_N}$.
            \item Zero non-free elements: $g_k\leftarrow 0\;\forall k\in\text{Free}\braro{\Lmat^{(t)}}$.
            \item Compute $\dvec$ as in \eqref{d CG}, with $\gvec$ in the role of $\nabla_\delta f_N$.\\
            \item  Update: $\delta^{(t+1)} = \textbf{max}\braro{\delta^{(t)} + \zeta \dvec,-\wvec}$, where $\zeta$ achieved by linesearch.
         \end{itemize}
        }
\end{algorithm}

\vfill
\newpage
\section{CONVERGENCE GUARANTEES: PROOF OF THEOREM \ref{Thm2}} \label{app:convergence}
In this section, we provide theoretical proof for the convergence of our method to a stationary point (Theorem \ref{Thm2}). We note that the MCP penalty is nonconvex; hence one typically gets to a local minimum using a monotone series of iterations. Some results in our convergence proof follow \citep{hsieh2014quic} (e.g., Lemma \ref{lem:bounds} and Lemma \ref{lem:decrease}). However, in the work \citep{hsieh2014quic}, the $\ell_1$ norm is used as a penalty, and the proof heavily relies on its convexity. Here, we had to adapt the proofs for these Lemmas to the concave MCP penalty, which is far from trivial and has not been done before, to the best of our knowledge. 

This section is organized as follows. In subsection \ref{sec10:setup} we define the considered setup of proximal methods for MCP regularized penalties. In subsection \ref{sec10:spectral} we provide some spectral bounds on our problem, which are needed for the proofs that follow. In subsection \ref{sec10:fixedpoint} we provide a lemma proving that our method is a fixed point iteration. In subsection \ref{sec10:decrease} we prove that the objective value sufficiently decreases between the iterations, which is the main point needed for the convergence proof. Lastly, in subsection \ref{sec10:convergence} we finalize the convergence proof given the previous auxiliary lemmas. 

\subsection{The Proximal Methods Setup}\label{sec10:setup}
Let us write the problem \eqref{minimization problem final} as follows
\beqna 
\label{eq:minimization problem F}
         \underset{\Lmat\in \mathcal{F}}{\text{minimize }} F(\Lmat) = 
           f(\Lmat) + \rho_{MCP}\braro{\Lmat;\lambda},
\eeqna
where $f$ is the smooth part of the objective, defined in \eqref{eq:smooth_f}, 
$\mathcal{F}$ and $\rho_{MCP}\braro{\Lmat;\lambda}$  are defined in \eqref{Lc - 9} and \eqref{r mcp}, respectively. In the setup of general proximal methods, at each iteration $t$, we minimize the penalized and constrained quadratic objective, as follows:
\begin{equation}
\begin{aligned}
\label{eq:minimization problem G}
   \Delta^{(t)}= \underset{\Delta \in \mathcal{S}^{p}}{\text{argmin}}
     \quad & G(\Delta) = f(\Lmat^{(t)}) + \big\langle\nabla f(\Lmat^{(t)}), \Delta\rangle + \frac{1}{2}\big\langle\Delta, \Delta\rangle_{\Mmat^{(t)}} + \rho_{MCP}\braro{\Lmat^{(t)}+\Delta;\lambda}\\
    \textrm{s.t.}
    \quad & (\Lmat^{(t)}+\Delta) \in \mathcal{F},
\end{aligned}
\end{equation}
where $\Mmat^{(t)}$ is some positive definite matrix that is an approximation of the true Hessian, or the true Hessian itself. In particular,  Proximal Gradient Descent uses $\Mmat = \Imat$, while proximal Newton will have $\Mmat = \nabla^2f(\Lmat) + \mbox{diag}(\mbox{vec}(\varepsilon))$, as in \eqref{Newton Problem} and Algorithm \ref{alg:OuterAlg}. Once the quadratic problem \eqref{eq:minimization problem G} is solved, we apply linesearch as in \eqref{eq:linesearch}, finding a step size $\alpha^{(t)}$ that ensures a reduction in the objective $F$. We now prove the convergence of this framework to our Laplacian estimation problem.   

\paragraph{Optimality Conditions}
We can state that by the definition of $\mathcal F$ our optimality condition is given by
\begin{equation}\label{eq:sub_zero_F}
\left\{
\begin{array}{ll}
0 \in \frac{\partial F(\Lmat^*)}{\partial \Lmat^*_{i,j}} + \mu_{i,j},& i>j\\
\Lmat^*_{j,i}=\Lmat^*_{i,j}\geq 0,& i\neq j\\
\Lmat^*_{i,i} = \sum_{j\neq i}-\Lmat^*_{i,j}& i=1,...,p
\end{array}\right.,
\end{equation}
where $\frac{\partial F(\Lmat^*)}{\partial \Lmat_{i,j}} = [\nabla f]_{i,j} +  \frac{\partial MCP(\Lmat^*_{ij};\gamma,\lambda)}{\partial \Lmat_{ij}}$ is the sub-differential of $F$ w.r.t $\Lmat_{i,j}$, $\mu_{i,j}\geq 0$ is an inequality Lagrange multiplier for the constraint $\Lmat_{i,j} \leq 0$, and the sub-differential of the MCP function in \eqref{MCP} is given by
\begin{equation}
    \label{gMCP}
    \frac{\partial MCP(x;\gamma,\lambda)}{\partial x}
    = \left\{\begin{array}{lc}
    \in [-\lambda,\lambda] & x = 0\\
    \lambda\cdot\text{sign}\braro{x} - \frac{x}{\gamma} 
 & 0 < |x| \leq \gamma\lambda\\
    0  & |x| > \gamma\lambda
\end{array}\right..
\end{equation}
If the conditions in \eqref{eq:sub_zero_F} are fulfilled, then $\Lmat^*$ is a stationary point of \eqref{minimization problem final}. Since the MCP penalty is nonconvex, we expect \eqref{minimization problem final} to have multiple minima, and we show that Algorithm \ref{alg:OuterAlg} converges to one of them. We initialize our algorithm with $\Lmat^{(0)}\in\mathcal{F}$ and keep $\Lmat^{(t)}\in\mathcal{F}$ in all iterations.

\subsection{Spectral Bounds}\label{sec10:spectral}
Let us first note that the $-\log\det(\Lmat+\Jmat)$ term in \eqref{minimization problem final} acts as a barrier function. That is, if \(\Lmat+\Jmat\) gets closer to being singular, then \(\det(\Lmat+\Jmat)\to0 \) and therefore, \(F(\Lmat)\to\infty\) following the $\log$ term. Therefore, we can say that as long as $F(\Lmat^{(t)})$ is bounded, there exists some $\beta_1 > 0$ for which \(\Lmat^{(t)}+\Jmat\succeq \beta_1 \Imat\) for all $t$. Furthermore, by \citep{ying2021minimax}, we have that for every Laplacian matrix $\Lmat$
\begin{equation}
\langle\Smat,\Lmat\rangle = \langle\mathbf{v},\mathbf{w}\rangle,
\end{equation}
where $\Lmat = \mathcal{P}\mathbf{w}, \;\;\mathbf{w}\geq 0$, and
\begin{equation}\label{eq:assumption}
\mathbf{v}_k = [\mathcal{P}^*\Smat]_k = \left[\mathcal{P}^*\frac{1}{n}\sum_{q=1}^n\mathbf{x}_q\mathbf{x}_q^T\right]_k = \frac{1}{n}\sum_{q=1}^{n}\left([\mathbf{x}_{q}]_i-[\mathbf{x}_q]_i\right)^2 > 0 
\end{equation}
holds with a very high probability, and $\mathcal{P}^*$ is the adjoint of $\mathcal{P}$ in \eqref{eq:adjoint P}. Following that, we have
\begin{equation}\label{eq:traceBound}
\langle \mathbf{v},\mathbf{w}\rangle \geq \min_j\{\mathbf{v}_j\}\sum_q{\mathbf{w}_q} = \frac{1}{2}\min_j\{\mathbf{v}_j\}\TR{\Lmat}.
\end{equation}
Now, for any $\Lmat\in\mathcal{F}$ such that $F(\Lmat) \leq F(\Lmat^{(0)})$ we have
$$
\TR{\Lmat\Smat} - \log|\Lmat|_+ +\rho\braro{\Lmat;\lambda} \leq F(\Lmat^{(0)}),
$$
and since $\rho\braro{\Lmat;\lambda} \geq 0$ we can write:
$$
\TR{\Lmat\Smat} \leq F(\Lmat^{(0)}) + \log|\Lmat|_+.
$$
Furthermore, using \eqref{eq:traceBound} and $\log|\Lmat|_+  \leq (p-1)\log(\lambda_{max}(\Lmat))$, we get
\be
\label{41}
\min_k\{\mathbf{v}_k\}\frac{1}{2}\lambda_{max}(\Lmat) \leq \min_k\{\mathbf{v}_k\}\frac{1}{2}\TR{\Lmat} \leq F(\Lmat^{(0)}) + (p-1)\log(\lambda_{max}(\Lmat)).
\ee
Now, since in \eqref{41} the left-had-side grows much faster than the right-hand-side, then there is and upper-bound on $\lambda_{max}(\Lmat)$, depending on $\Smat$ and $\Lmat^{(0)}$, assuming \eqref{eq:assumption}. Hence, we can conclude that if the iterations $\{\Lmat^{(t)}\}$ are monotonically decreasing starting from $\Lmat^{(0)}\in\mathcal{F}$, there exist upper and lower bounds
\begin{equation}
\beta_1\Imat \preceq \Lmat^{(t)}+\Jmat \preceq \beta_2\Imat.
\end{equation}
This also means that the level set $\mathcal{R} = \{\Lmat\in\mathcal{F}: F(\Lmat) \leq F(\Lmat^{(0)})\}$ is compact, and the Hessian is bounded, i.e., $\nabla^2f(\Lmat)=(\Lmat+\Jmat)^{-1}\otimes (\Lmat+\Jmat)^{-1} \preceq \frac{1}{\beta_1^2}\Imat = \theta\Imat$, where $\theta>0$. The last upper bound comes from the properties of eigenvalues of Kronecker products. 

We summarize the conclusions above in the following lemma:
\begin{lemma}\label{lem:bounds}
Assume that the iterations $\Lmat^{(t)}$ are monotonically decreasing with respect to the objective in \eqref{eq:minimization problem F}, starting from a feasible $\Lmat^{(0)}$. Then there exist constants $\beta_1, \beta_2$ such that $\beta_1\Imat \preceq \Lmat^{(t)}+\Jmat \preceq \beta_2\Imat$, and the level-set $\mathcal{R} = \{\Lmat\in\mathcal{F}: F(\Lmat) \leq F(\Lmat^{(0)})\}$ is compact. Also, the Hessian is bounded:  $\nabla^2f(\Lmat) \preceq  \frac{1}{\beta_1^2}\Imat = \theta\Imat$
\end{lemma}

\subsection{Fixed Point Iteration}\label{sec10:fixedpoint}

The following lemma shows that if $\Lmat^{(t)}$ is feasible, and the quadratic minimization \eqref{eq:minimization problem G} ends with $\Delta^{(t)} = 0$, then $\Lmat^{(t)}$ is a minimum point of \eqref{eq:minimization problem F}. 

\begin{lemma}\label{lem:fixed_point}
Assume that $\Lmat^{(t)}\in\mathcal{F}$. If $\Delta^{(t)} = 0$ is a minimizer of $G(\Delta)$ in  \eqref{eq:minimization problem G}, then $\Lmat^{(t+1)}=\Lmat^{(t)}$ is a minimizer of $F(\Lmat)$ in \eqref{eq:minimization problem F}.
\end{lemma}
\begin{proof}
Since $\Delta^{(t)}$ is a minimizer of \eqref{eq:minimization problem G}, the following optimality conditions are held
\begin{equation}\label{eq:sub_zero_G}
\left\{
\begin{array}{ll}
0 \in \frac{\partial G(\Delta^{(t)})}{\partial \Delta_{i,j}} + \mu_{i,j},& i>j\\
\Lmat^{(t)}_{j,i} + \Delta^{(t)}_{j,i}=\Lmat^{(t)}_{i,j} + \Delta^{(t)}_{i,j}\geq 0,& i\neq j\\
\Lmat^{(t)}_{i,i} + \Delta^{(t)}_{i,i} = \sum_{j\neq i}-(\Lmat^{(t)}_{i,j} + \Delta^{(t)}_{i,j})& i=1,...,p
\end{array}\right.,
\end{equation}
where $\frac{\partial G(\Delta)}{\partial \Delta_{i,j}} = [\nabla f]_{i,j} + [\langle\Mmat\Delta\rangle]_{i,j}+ \frac{\partial MCP(\Lmat^{(t)}_{i,j}+\Delta_{i,j};\gamma,\lambda)}{\partial \Delta_{ij}}$, and $\langle\Mmat\Delta\rangle$ denotes a matrix vector multiplication of $\Mmat$ with the vectorized $\Delta$, i.e., $\mbox{mat}(\Mmat\cdot\mbox{vec}(\Delta))$. Placing $\Delta^{(t)} = 0$ in \eqref{eq:sub_zero_G} yields the same conditions as in \eqref{eq:sub_zero_F} regardless of the specific choice of the matrix $\Mmat\succ 0$, hence $\Lmat^{(t)}$ is a stationary point of \eqref{eq:minimization problem F}.
\end{proof}

\subsection{Sufficient Decrease in Objective}\label{sec10:decrease}

We now show that we have a decrease in the objective $F$ following the minimization in \eqref{eq:minimization problem G} and the linesearch in \eqref{eq:linesearch}. This lemma is the key to our convergence proof. 
\begin{lemma}\label{lem:decrease}
Assume that the Hessian is bounded, $\nabla^2{f} \preceq  \theta\Imat$ for $\theta>0$. Also assume that the iterates $\{\Lmat^{(t)}\}$ are defined by \eqref{eq:minimization problem G} followed by a linesearch \eqref{eq:linesearch}, starting from $\Lmat^{(0)}\in\mathcal{F}$, with $\Mmat^{(t)}$ satisfying $\Mmat^{(t)} \succeq \lambda^{\Mmat}_{min}\Imat \succ \gamma^{-1}\Imat$, where $\lambda^{\Mmat}_{min}$ is a constant and $\gamma$ is the MCP parameter. Then
\begin{equation}
\label{eq:monotonicityNeed}
F(\Lmat^{(t+1)}) - F(\Lmat^{(t)}) \leq -c\cdot \|\Delta^{(t)}\|_F^2,
\end{equation}
where $c>0$. Furthermore, the linesearch parameter $\alpha^{(t)}$ in \eqref{eq:linesearch} can be chosen to be bounded away from zero, i.e., $\alpha^{(t)} \geq \alpha_{min} > 0$.
\end{lemma}

\begin{proof}
For $\Delta^{(t)}$, the solution of \eqref{eq:minimization problem G}, and some linesearch parameter $0<\alpha<1$ we have:
\begin{equation}
\begin{array}{ll}
\big\langle\nabla f(\Lmat^{(t)}), \Delta\big\rangle + \frac{1}{2}\big\langle\Delta^{(t)}, \Delta^{(t)}\rangle_{\Mmat^{(t)}} + \rho_{MCP}\braro{\Lmat^{(t)}+\Delta^{(t)};\lambda} \\\\ \quad \quad \leq  
\langle\nabla f(\Lmat^{(t)}), \alpha\Delta^{(t)}\rangle + \frac{1}{2}\langle\alpha\Delta^{(t)}, \alpha\Delta^{(t)}\rangle_{\Mmat^{(t)}} + \rho_{MCP}\braro{\Lmat^{(t)}+\alpha\Delta^{(t)};\lambda}.
\end{array}
\end{equation}
That is because the objective in \eqref{eq:minimization problem G} is lower for $\Delta^{(t)}$ than for $\alpha\Delta^{(t)}$ (because $\Delta^{(t)}$ is assumed to be the minimum). Also, since $\mathcal{F}$ is convex, then $\Lmat^{(t)} + \alpha\Delta\in \mathcal{F}$. Thus, we have
\begin{equation}
(1-\alpha)\big\langle\nabla f(\Lmat^{(t)}), \Delta\big\rangle + (1-\alpha^2)\frac{1}{2}\big\langle\Delta, \Delta\rangle_{\Mmat^{(t)}} + \rho_{MCP}\braro{\Lmat^{(t)}+\Delta;\lambda} - \rho_{MCP}\braro{\Lmat^{(t)}+\alpha\Delta;\lambda}\leq 0.
\end{equation}
After dividing by $(1-\alpha) > 0$ we get:
\begin{equation}\label{eq:upper_boud_first}
\big\langle\nabla f(\Lmat^{(t)}), \Delta\big\rangle \leq -(1+\alpha)\frac{1}{2}\big\langle\Delta, \Delta\rangle_{\Mmat^{(t)}} - \frac{1}{1-\alpha}\left(\rho_{MCP}\braro{\Lmat^{(t)}+\Delta;\lambda} - \rho_{MCP}\braro{\Lmat^{(t)}+\alpha\Delta;\lambda}\right).
\end{equation}

Now, by the definition of $MCP$ in \eqref{MCP}, for every $x,y$ such that $sign(x) = sign(x+y)$ the function is twice differentiable in the interval $(x,x+y)$, and we can write 
\begin{equation}\label{eq:MCP_Approx}
MCP(x+y) = MCP(x) +MCP'(x) y - \frac{1}{2\gamma}(\omega y)^2, 
\end{equation}
where $0\leq\omega\leq 1$. 
Given that the feasible $\mathcal{F}$ states that $\Lmat^{(t)}\leq 0$ and $\Lmat^{(t)}+\Delta\leq 0$ at all times, we can write
\begin{equation}\label{eq:Omega}
\rho_{MCP}\braro{\Lmat^{(t)}+\Delta;\lambda} = \rho_{MCP}\braro{\Lmat^{(t)};\lambda} + \langle \nabla\rho_{MCP},\Delta\rangle - \frac{1}{2\gamma}\langle \Omega\odot\Delta, \Omega\odot\Delta\rangle
\end{equation}
where $0\leq\Omega \leq 1$ is a matrix of the $\omega$ values as in \eqref{eq:MCP_Approx}, and $\odot$ is an element-wise product. Similarly, let 
\begin{equation}\label{eq:Omega_alpha}
\rho_{MCP}\braro{\Lmat^{(t)}+\alpha\Delta;\lambda} = \rho_{MCP}\braro{\Lmat^{(t)};\lambda} + \langle \nabla\rho_{MCP},\alpha\Delta\rangle - \frac{1}{2\gamma}\langle \Omega_\alpha\odot\alpha\Delta, \Omega_\alpha\odot\alpha\Delta\rangle,
\end{equation}
where $0\leq\Omega_\alpha \leq 1$. 
To show a decrease in the function $F$ following the solution of \eqref{eq:minimization problem G} and a linesearch, we write
\begin{align}
F(\Lmat^{(t)} + \alpha \Delta^{(t)}) - F(\Lmat^{(t)})  =& f(\Lmat^{(t)} + \alpha \Delta^{(t)}) + \rho_{MCP}\braro{\Lmat^{(t)}+\alpha\Delta;\lambda}  - f(\Lmat^{(t)})  - \rho_{MCP}\braro{\Lmat^{(t)};\lambda} \\
=& \label{eq:Temp50}\langle\nabla f(\Lmat^{(t)}), \alpha\Delta^{(t)}\rangle + \frac{1}{2}\langle\alpha\Delta^{(t)}, \alpha\Delta^{(t)}\rangle_{\nabla^2f(\Lmat^{(t)})} + O(\|\Delta\|^3) \\ \nonumber
& + \rho_{MCP}\braro{\Lmat^{(t)}+\alpha\Delta;\lambda}
- \rho_{MCP}\braro{\Lmat^{(t)};\lambda}\\
\leq & \label{eq:Temp51} -\frac{\alpha(1+\alpha)}{2}\big\langle\Delta, \Delta\rangle_{\Mmat^{(t)}} + \frac{1}{2}\langle\alpha\Delta^{(t)}, \alpha\Delta^{(t)}\rangle_{\nabla^2f(\Lmat^{(t)})} + O(\|\Delta\|^3)\\\nonumber
& - \frac{\alpha}{1-\alpha}\left(\rho_{MCP}\braro{\Lmat^{(t)}+\Delta;\lambda} - \rho_{MCP}\braro{\Lmat^{(t)}+\alpha\Delta;\lambda}\right) \\ \nonumber
 &+ \rho_{MCP}\braro{\Lmat^{(t)}+\alpha\Delta;\lambda}
- \rho_{MCP}\braro{\Lmat^{(t)};\lambda}.
\end{align}
Equation \eqref{eq:Temp50} is a simple Taylor expansion, and the inequality in \eqref{eq:Temp51} stems from \eqref{eq:upper_boud_first}. At this point, we will focus on the last two lines of \eqref{eq:Temp51}:
\begin{align}
& - \frac{\alpha}{1-\alpha}\left(\rho_{MCP}\braro{\Lmat^{(t)}+\Delta;\lambda} - \rho_{MCP}\braro{\Lmat^{(t)}+\alpha\Delta;\lambda}\right) 
 + \rho_{MCP}\braro{\Lmat^{(t)}+\alpha\Delta;\lambda}
- \rho_{MCP}\braro{\Lmat^{(t)};\lambda}\\
&= \frac{1}{1-\alpha}\left((1-\alpha+\alpha)\rho_{MCP}\braro{\Lmat^{(t)}+\alpha\Delta;\lambda} - \alpha\rho_{MCP}\braro{\Lmat^{(t)}+\Delta;\lambda} \right) - \rho_{MCP}\braro{\Lmat^{(t)};\lambda}\\
&= \frac{1}{1-\alpha}\left((1-\alpha)\rho_{MCP}\braro{\Lmat^{(t)};\lambda}  - \frac{1}{2\gamma}\langle \alpha^2\Omega_\alpha^2 - \alpha\Omega^2\odot\Delta, \Delta\rangle\right) - \rho_{MCP}\braro{\Lmat^{(t)};\lambda}\\
&= \frac{1}{1-\alpha}\left(- \frac{1}{2\gamma}\langle \alpha^2\Omega_\alpha^2 - \alpha\Omega^2\odot\Delta, \Delta\rangle\right)
\leq \frac{\alpha}{2\gamma}\langle\max\{\Omega_\alpha^2, \Omega^2\}\odot\Delta, \Delta\rangle \leq \frac{\alpha}{2\gamma}\langle\Delta, \Delta\rangle,\label{eq:OmegaAdvantage}
\end{align}
where $\Omega_\alpha$ and $\Omega$ are defined in \eqref{eq:Omega}-\eqref{eq:Omega_alpha}. Going back to the bound on the decrease in $F$ in \eqref{eq:Temp51} we get:
\begin{align}
F(\Lmat^{(t)} + \alpha \Delta^{(t)}) - F(\Lmat^{(t)})  \leq&   -\frac{\alpha(1+\alpha)}{2}\big\langle\Delta, \Delta\rangle_{\Mmat^{(t)}} + \frac{1}{2}\langle\alpha\Delta^{(t)}, \alpha\Delta^{(t)}\rangle_{\nabla^2f(\Lmat^{(t)})} +\frac{\alpha}{2\gamma}\langle\Delta, \Delta\rangle+ O(\|\Delta\|^3)\\ \label{eq:upper decrease}
=& -\frac{\alpha}{2}\langle\Delta,\Delta\rangle_{(\Mmat^{(t)} - \gamma^{-1}\Imat)} - \frac{\alpha^2}{2}\langle\Delta,\Delta\rangle_{(\Mmat^{(t)} - \nabla^2f(\Lmat^{(t)}))} + O(\|\Delta\|^3).
\end{align}
Two conditions are needed to keep the above bound negative and guarantee a decrease in the objective. One is by having $\Mmat^{(t)}\succ \gamma^{-1}\Imat$, which is easy to guarantee, e.g., via the choice of $\varepsilon$ in our case. The other condition, assuming that we have $\Mmat^{(t)} \succ \gamma^{-1}\Imat$, can be written as 
$$\alpha(\Mmat-\gamma^{-1}) + \alpha^2\nabla^2f(\Lmat^{(t)})\succ 0,$$
and then we have
\begin{equation}\label{eq:monotonicity}
F(\Lmat^{(t)} + \alpha \Delta^{(t)}) - F(\Lmat^{(t)}) \leq -\alpha(\lambda_{min}^{\Mmat} - \frac{1}{\gamma} - \alpha \theta) \|\Delta^{(t)}\|_F^2.
\end{equation}
where $\theta$ is the upper bound on the Hessian, and $\lambda_{min}^{\Mmat} > \gamma^{-1}$ is the smallest eigenvalue of $\Mmat$. This value is negative for any
$
0\leq\alpha_{min} < \frac{\lambda_{min}^{\Mmat} - \gamma^{-1}}{\theta}
$.
\end{proof}

The result above is relevant to Algorithm \ref{alg:OuterAlg}, where we have $\Mmat^{(t)} = \nabla^2 f(\Lmat^{(t)}) + \mbox{diag}(\epsilon^2)$. Hence, the second term in \eqref{eq:upper decrease} almost vanishes if $\epsilon$ is small. On the other hand, $\epsilon$ allows us to make sure that the first term in \eqref{eq:upper decrease} remains negative. We note that one can also monitor which entries are in the concave region of the MCP (smaller than $\gamma\lambda$ in magnitude) and set some value of $\varepsilon_{i,j}$ for them only, following \eqref{eq:OmegaAdvantage}.  

The analysis above essentially is also suitable for Algorithm \ref{algorithm for Newton problem} if we set $\beta_t=0$, i.e., use a projected quasi-Newton method with $\Dmat$ as a diagonal approximation to the Hessian.
This is because in Algorithm \ref{algorithm for Newton problem}, the Newton problem is also a smooth objective with an MCP prior and positivity constraints, and the change of variables to $\delta$ is only for convenience of symmetry and does not change the problem. 

\subsection{Convergence}\label{sec10:convergence}

Before stating the main convergence theorem, we prove the following auxiliary lemma:

\begin{lemma} \label{lem:subseries} 
Let $\{\Lmat^{(t)}\}$ be series of points produced by \eqref{eq:minimization problem G} followed by the linesearch \eqref{eq:linesearch}, with $\lambda^{\Mmat}_{max}\Imat\preceq \Mmat^{(t)} \preceq \lambda^{\Mmat}_{min}\Imat \preceq 0$. Let $\{\Lmat^{(t_j)}\}$ be any converging subseries of $\{\Lmat^{(t)}\}$ and let $\bar\Lmat$ be its limit. Then $\bar\Lmat$ is a stationary point of $F$ in \eqref{minimization problem final}. 
\end{lemma}

\begin{proof}
Since the sub-series \(\{\Lmat^{(t_j)}\}\) convergences to \(\bar{\Lmat}\), then \(\{F(\Lmat^{(t_j)})\}\) convergences to \(F(\bar{\Lmat})\). Therefore, \(\{F(\Lmat^{(t_j+1)})-F(\Lmat^{(t_j)})\}\rightarrow0\), which implies following \eqref{eq:monotonicity} that \(||\Lmat^{(t_j)} - \Lmat^{(t_j+1)}||^2_F \rightarrow 0\) and \(\lim_{j\rightarrow \infty}\Lmat^{(t_j+1)}=\bar{\Lmat}\). According to \eqref{eq:minimization problem G} and the linesearch \eqref{eq:linesearch}, \(\Lmat^{(t+1)} = \Lmat^{(t)}+\alpha^{(t)}\Delta^{(t)}\) and because, \(||\Lmat^{(t_j+1)} - \Lmat^{(t_j)}||^2_F \rightarrow 0\) and $\alpha \geq \alpha_{min} > 0$, this implies \(||\Delta^{(t)}||^2_F\rightarrow 0\). We know that \(\Delta^{(t)}\) satisfies \eqref{eq:sub_zero_G}, and by taking the limit as \(j\rightarrow\infty\) and using Lemma \ref{lem:fixed_point} we get that \(\bar{\Lmat}\) is a stationary point of \eqref{minimization problem final}.
\end{proof}

Now, after providing the auxiliary lemmas above, we are ready to state the convergence theorem.

\textbf{Theorem 2.}
\textit{Let $\{\Lmat^{(t)}\}$ be a series of points produced by a proximal quadratic minimization, constrained by $\mathcal{F}$ with some SPD matrices $0\prec\lambda^{\Mmat}_{min}\Imat\preceq \Mmat^{(t)} \preceq \lambda^{\Mmat}_{max}\Imat$ as the Hessian (as in \eqref{eq:minimization problem G}), followed by the linesearch \eqref{eq:linesearch}, starting from $\Lmat^{(0)}\in\mathcal{F}$. Then any limit point $\bar \Lmat$ of $\{\Lmat^{(t)}\}$ is a stationary point of \eqref{minimization problem final}.}

\begin{proof}
By Lemma \ref{lem:decrease}, the series $F(\Lmat^{(t)})$ is monotonically decreasing, hence it is also convergent as $F$ is bounded from below. Since the level set $\mathcal{R}$ is compact by Lemma \ref{lem:bounds}, and $\Lmat^{(t)}$ is bounded in $\mathcal{R}$, there is a converging subseries $\{\Lmat^{(t_j)}\}$, which let $\bar\Lmat$ be its limit point. By Lemma \ref{lem:subseries}, $\bar\Lmat$ is a stationary point of $F$ in \eqref{minimization problem final}. Since $F$ is continuous then $F(\Lmat^{(t_j)})\longrightarrow F(\bar \Lmat)$. Since the limit of $\{\Lmat^{(t)}\}$ and $F(\Lmat^{(t)})$ equal to that of any of their corresponding subseries, then $\{\Lmat^{(t)}\}\longrightarrow \bar\Lmat$ and $F(\Lmat^{(t)})\longrightarrow F(\bar \Lmat)$. 
\end{proof}

\section{RESULTS USING REAL DATA }
{\textcolor{black}{
In this subsection, we examine several gene expression datasets that are commonly used in research related to model selection, classification, and graphical models. These datasets are genetic regulatory networks,  in which individual genes are represented as nodes in a graph, with the edges denoting the conditional dependencies between their expression profiles.  It is important to note that these datasets can consist of thousands of genes while the number of samples is limited.  This requires sparsity assumptions to enable the topology identification. Detailed information regarding the datasets, namely Lymph, Arabidopsis, Leukemia, and Hereditary BC datasets, can be found in \citep{li2010inexact}.
As there is no ground truth in this case, we show convergence plots of the different methods in Fig. \ref{fig:bio}.
It can be seen that for these real data simulations, the proposed method converges significantly faster than the other methods.
}}

\begin{figure*}
    \centering
    \includegraphics[width=0.24\textwidth]{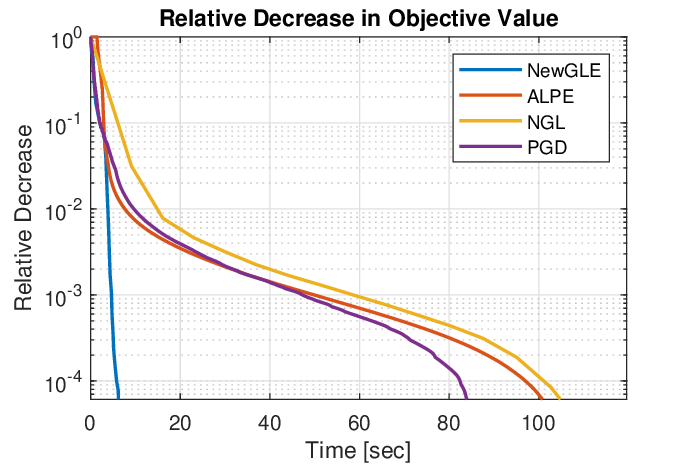}
    \includegraphics[width=0.24\textwidth]{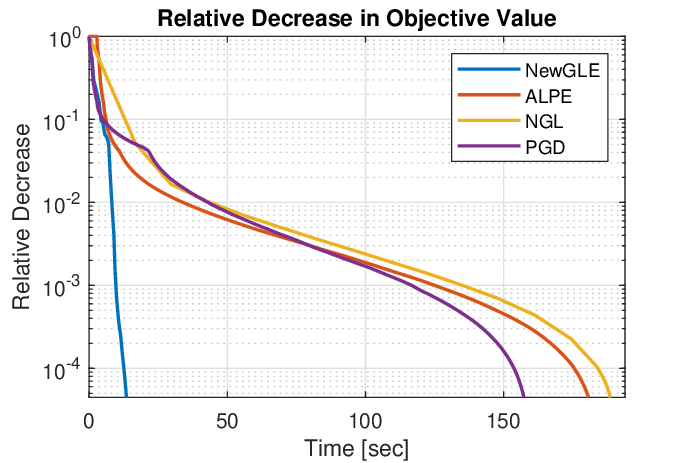}
    \includegraphics[width=0.24\textwidth]{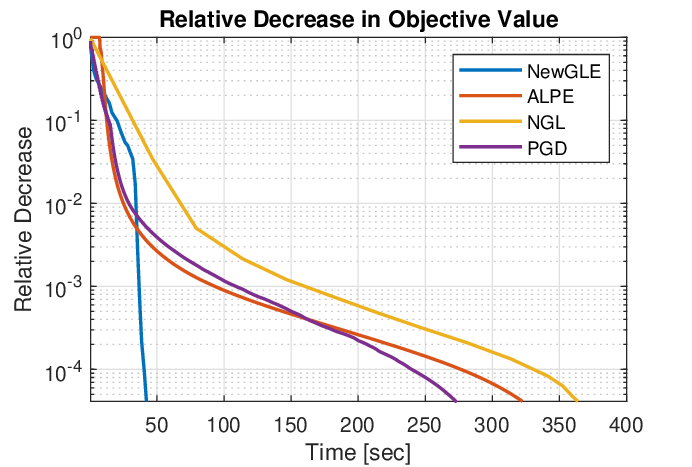}
    \includegraphics[width=0.24\textwidth]{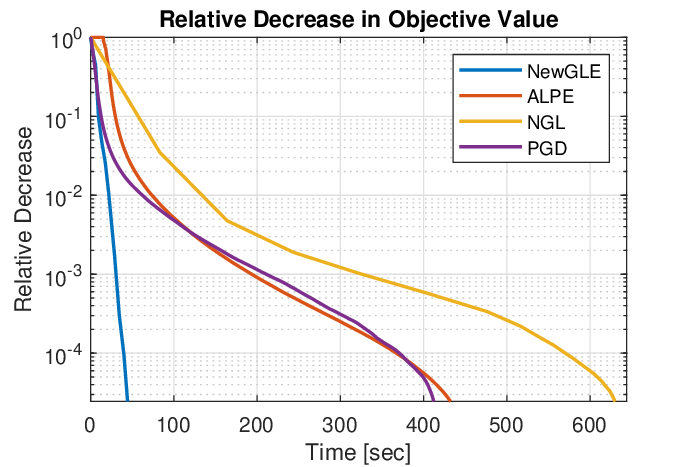}
    \caption{{\small Convergence plots 
    for the Lymph, Arabidopsis, Leukemia, and Hereditary BC datasets (left to right). The relative decrease is measured by $\frac{|F(\Lmat^{(t)})-F(\Lmat^{(0)})|}{ |F(\Lmat^*)-F(\Lmat^{(0)})|}$}}
    \label{fig:bio}
    \end{figure*}
\end{document}

%% file: Definitions.tex
\usepackage{setspace}
\usepackage[normalem]{ulem}
\usepackage{amsmath,amssymb,amsthm}
\usepackage{enumerate}
\usepackage{float}
\usepackage{mathtools}
\usepackage{multirow}
\usepackage{graphicx}
\usepackage[ruled,vlined]{algorithm2e}
\usepackage{xcolor}
\usepackage[utf8]{inputenc}
\usepackage{lscape}
\usepackage{epsfig,subcaption}
\usepackage{listings}
\usepackage{color}
\usepackage{bbm}
\usepackage{hyperref}

\newtheorem{theorem}{Theorem}

\newtheorem{lemma}{Lemma}

\newtheorem{definition}{Definition}

\SetKwInput{KwInput}{Input}                
\SetKwInput{KwOutput}{Output}              
\SetKwInput{KwInit}{Initialization}        
\definecolor{mygreen}{RGB}{28,172,0} 
\definecolor{mylilas}{RGB}{170,55,241}
\usepackage[framed,numbered,autolinebreaks,useliterate]{mcode}

\newcommand{\brasq}[1]{\left[{#1}\right]}
\newcommand{\braro}[1]{\left({#1}\right)}
\newcommand{\bracu}[1]{\left\{{#1}\right\}}
\newcommand{\TR}[1]{\Tr{\left({#1}\right)}}

\newcommand{\Lx}[1]{\mathcal{P}\braro{{#1}}}
\newcommand{\mycomment}[1]{}

\DeclareUnicodeCharacter{0301}{\'{i}}

%% file: myshorts.tex
\newcommand{\dvec}{{\bf{d}}}

\newcommand{\wvec}{{\bf{w}}}
\newcommand{\xvec}{{\bf{x}}}

\newcommand{\vvec}{{\bf{v}}}
\newcommand{\gvec}{{\bf{g}}}

\newcommand{\onevec}{{\bf{1}}}
\newcommand{\zerovec}{{\bf{0}}}

\newcommand{\Amat}{{\bf{A}}}

\newcommand{\Dmat}{{\bf{D}}}

\newcommand{\Hmat}{{\bf{H}}}
\newcommand{\Jmat}{{\bf{J}}}
\newcommand{\Imat}{{\bf{I}}}
\newcommand{\Lmat}{{\bf{L}}}
\newcommand{\Mmat}{{\bf{M}}}

\newcommand{\Qmat}{{\bf{Q}}}

\newcommand{\Smat}{{\bf{S}}}

\newcommand{\Umat}{{\bf{U}}}

\newcommand{\Wmat}{{\bf{W}}}

\newcommand{\Ymat}{{\bf{Y}}}

\newcommand{\define}{\stackrel{\triangle}{=}}

\newcommand{\be}{\begin{equation}}
\newcommand{\ee}{\end{equation}}
\newcommand{\beqna}{\begin{eqnarray}}
\newcommand{\eeqna}{\end{eqnarray}}
